\documentclass{article}

\usepackage{arxiv}

\usepackage[utf8]{inputenc} 
\usepackage[T1]{fontenc}    
\usepackage{hyperref}       
\usepackage{url}            
\usepackage{booktabs}       
\usepackage{amsfonts}       
\usepackage{nicefrac}       
\usepackage{microtype}      
\usepackage{lipsum}         
\usepackage{graphicx}
\usepackage{natbib}
\usepackage{doi}
\usepackage[dvipsnames,table]{xcolor}         
\usepackage{authblk}

\setcitestyle{numbers,square,comma}
\usepackage{graphicx}
\usepackage{comment}
\usepackage{bm}
\usepackage{wrapfig}
\usepackage{multirow}
\usepackage{enumitem}
\usepackage{tcolorbox}
\usepackage{amsthm}
\usepackage{mathtools}
\usepackage{subcaption}
\usepackage{amssymb}
\usepackage{cleveref}       
\tcbset{colback=gray!1!white, colframe=gray!50!black, fonttitle=\bfseries, width=\linewidth}
\theoremstyle{plain}
\newtheorem{theorem}{Theorem}[section]

\newtheorem{assumption}[theorem]{Assumption}
\newtheorem{example}[theorem]{Example}
\usepackage{bm, amsmath}

\DeclareMathOperator*{\argmin}{arg\,min}

\def\bigO{\mathcal{O}}

\def\1{\bm{1}}
\def\0{\bm{0}}

\newcommand{\E}{\mathbb{E}}
\newcommand{\given}{\,\middle|\,}
\newcommand{\Var}{\mathrm{Var}}

\def\Figref#1{Fig.~\ref{#1}}

\def\Secref#1{Sec.~\ref{#1}}
\def\eqref#1{Eq.~\ref{#1}}
\def\Eqref#1{Eq.~\ref{#1}}

\def\tabref#1{Tab.~\ref{#1}}
\def\Tabref#1{Tab.~\ref{#1}}

\def\Appref#1{App.~\ref{#1}}

\newcommand{\ie}{\textnormal{i.e., }}
\newcommand{\eg}{\textnormal{e.g., }}
\newcommand{\wrt}{\textnormal{w.r.t.\ }}

\def\vb{{\bm{b}}}

\def\ve{{\bm{e}}}

\def\vh{{\bm{h}}}

\def\vm{{\bm{m}}}

\def\vr{{\bm{r}}}

\def\vx{{\bm{x}}}

\def\vz{{\bm{z}}}

\def\mW{{\bm{W}}}

\DeclareMathAlphabet{\mathsfit}{\encodingdefault}{\sfdefault}{m}{sl}
\SetMathAlphabet{\mathsfit}{bold}{\encodingdefault}{\sfdefault}{bx}{n}

\def\gC{{\mathcal{C}}}

\def\gG{{\mathcal{G}}}

\def\gL{{\mathcal{L}}}

\def\gN{{\mathcal{N}}}

\def\gP{{\mathcal{P}}}

\def\gU{{\mathcal{U}}}

\def\sC{{\mathbb{C}}}

\def\sM{{\mathbb{M}}}
\def\sN{{\mathbb{N}}}

\def\sP{{\mathbb{P}}}

\def\sR{{\mathbb{R}}}

\usepackage[acronym, nohypertypes={acronym}]{glossaries}    

\newacronym{sota}{SOTA}{state-of-the-art}
\newacronym{mlp}{MLP}{\textit{multilayer perceptron}}
\newacronym{rnn}{RNN}{\textit{recurrent neural network}}
\newacronym{tcn}{TCN}{\textit{temporal convolutional network}}
\newacronym{gru}{GRU}{\textit{gated recurrent unit}}
\newacronym{lru}{LRU}{\textit{linear recurrent unit}}
\newacronym{llm}{LLM}{\textit{large language model}}
\newacronym{nlp}{NLP}{\textit{natural language processing}}
\newacronym{icl}{ICL}{\textit{in-context learning}}

\newacronym{mae}{MAE}{\textit{mean absolute error}}
\newacronym{mse}{MSE}{\textit{mean squared error}}
\newacronym{rmse}{RMSE}{\textit{root mean squared error}}

\newacronym{flops}{FLOPs}{\textit{FLoating Point Operations per second}}
\newacronym{nar}{NAR}{\textit{nonlinear autoregressive}}
\newacronym{ar}{AR}{\textit{autoregressive}}

\newacronym{gpi}{GPI}{\textit{generative process identification}}
\newacronym{cf}{CF}{\textit{conditional forecasting}}

\glsdisablehyper

\newglossaryentry{cer}{name=CEREn,description=}
\newglossaryentry{climate-t}{name=ClmUS-T,description=}
\newglossaryentry{climate-r}{name=ClmUS-R,description=}
\newglossaryentry{largest}{name=LargeST-SD,description=}
\graphicspath{ {./figures/} }

\title{Why Do Time Series Models Need Long Context Windows?}
\date{}

\author[1]{%
  Luca Butera
}
\author[1]{%
  Giovanni De Felice
}
\author[2]{%
  Andrea Cini
}
\author[1,3]{%
  Cesare Alippi
}
\affil[1]{Università della Svizzera Italiana}
\affil[2]{EPFL}
\affil[3]{Politecnico di Milano}

\renewcommand{\headeright}{Preprint}
\renewcommand{\undertitle}{Preprint}
\renewcommand{\shorttitle}{Why Do Time Series Models Need Long Context Windows?}

\begin{document}

\maketitle

\begin{abstract}
  Modern deep learning models for forecasting \textit{groups of time series} rely on increasingly longer observation windows. However, the benefit of increasing the window size is often simply attributed to capturing long-range dependencies, and broader discussion on how \textit{global} forecasting models leverage input observations has been limited. In this paper, we show that forecasting groups of time series involves two objectives: (i)~\textit{generative process identification}~(GPI), i.e., inferring the specific process generating the input sequence, and (ii)~\textit{conditional forecasting}~(CF), i.e., predicting future values given input observations. From this perspective, optimal predictions can be interpreted as an average over plausible data-generating processes, weighted by their likelihood given the input window. This suggests another explanation for the benefits of long context windows: they reduce the uncertainty about which specific process is generating the input time series during operation. We prove that even for processes with memory length $P$, an input window size strictly larger than $P$ is \textit{necessary} to achieve the minimum attainable error. Finally, we show how decoupling GPI and CF can improve computational scalability without compromising accuracy. Experiments on synthetic and real-world data validate our insights and their relevance for designing forecasting architectures.
\end{abstract}

\section{Introduction}\label{sec:introduction}
\begin{wrapfigure}[19]{r}{0.51\textwidth}
    \centering
    \vspace*{-0.15cm}
    \includegraphics[width=\linewidth]{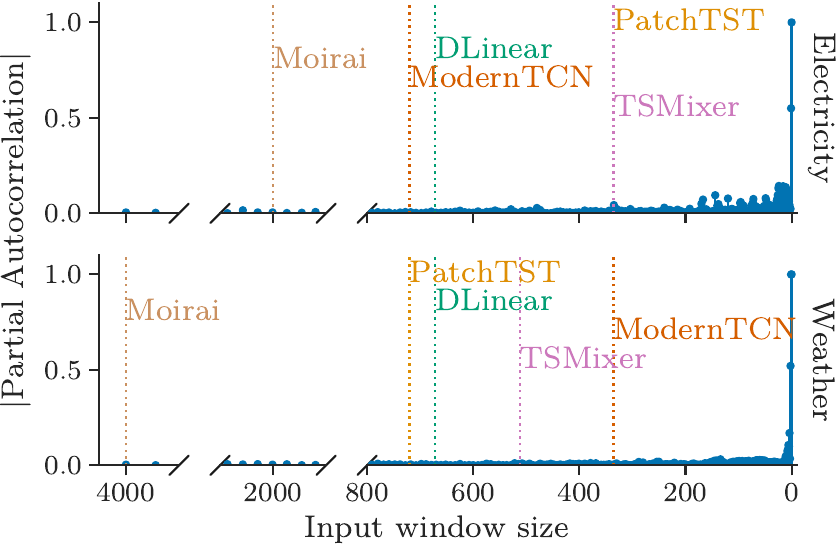}
    \vspace*{-0.55cm}
    \caption{Partial autocorrelation at different time lags in two widely used forecasting benchmarks~\citep{wu2021autoformer}. Vertical lines indicate the best window length for popular forecasting models.}
    \label{fig:pacf_vs_window}
\end{wrapfigure}
Deep learning approaches have become primary tools in forecasting large collections of time series~\citep{benidis2022deep, wang2024deep}. While the standard statistical approach is to independently fit a \textit{local} model for each target time series~\citep{box1968some, degooijer2006years}, more recent advancements show that a single, shared, \textit{global} model, can be effective across multiple series~\citep{hewamalage2021recurrent, salinas2020deepar, makridakis2020m4, montero2021principles}. 
In this context, recent research has focused on architectures capable of handling increasingly longer observation windows~\citep{haoyietal2021informer,nie2022time,zhou2022fedformer,liu2022pyraformer}, typically attributing performance improvements to the ability to capture long-range dependencies. 
While global models have been used \textit{transductively} to forecast future observations of the training time series, several time series \textit{foundation} models have been recently introduced~\citep{das2024decoder, woo2024unified, goswami2024moment, liang2024foundation}. These are global models that aim to \textit{inductively} forecast previously unseen time series.
For this purpose, they are trained on vast corpora of time series coming from different domains, and usually rely on very long input windows and high model complexity. 
Interestingly, as \Figref{fig:pacf_vs_window} illustrates, in recent forecasting architectures, the window size selected with hyperparameter tuning is considerably larger than the number of time lags one would consider relevant, \eg by looking at autocorrelation. 
Indeed, \citet{montero2021principles} observed that transductive global models inherently require longer windows and higher capacity than local counterparts to guarantee the same performance. 
Beyond time series forecasting, there have been extensive efforts to understand how models trained on multiple tasks transfer to a new one given sufficient context~\citep{xie2022an,akyurek2023what,li2023transformers,zhang2025whatandhow}, a capability known as \gls{icl}~\citep{dong2024survey}. Nonetheless, discussion on how global forecasting models leverage input data beyond modeling temporal dependencies has been limited. 

In this paper, we aim to fill this void and show that training a model to minimize forecasting error over groups of time series can be seen as involving two different tasks: (i)~\gls{gpi}, \ie inferring the generative process that produced the input observations, and (ii)~\gls{cf}, \ie predicting the next observations conditioned on the past. While most architectures do not explicitly decouple these tasks, as we will show, they can be viewed as inherent in how global models are trained and used in practice, regardless of whether they are used \textit{transductively} or \textit{inductively}. 
We show that optimal predictions can be interpreted as a \textit{posterior average} over the possible generative processes, \ie an average of forecasts for all plausible processes, each weighted by the estimated likelihood of that process generating the input sequence. This has important theoretical and methodological implications. 
Our main novel contributions and findings are as follows.
\begin{itemize}[leftmargin=0.5em, itemindent=0.25em, labelsep=0.25em, itemsep=0cm, topsep=0cm, parsep=0.1em]
\item \textbf{Problem formulation.} 
We formulate the task of learning a global forecasting model as learning from data generated by multiple generative processes from a common parametric family (\Secref{sec:setting}). This formulation provides a framework for characterizing time series foundation models and understanding the associated learning problem.

\item \textbf{The role and cost of \gls{gpi}.}
We provide theoretical and empirical evidence showing that, especially for foundation models, the input length and model capacity required for accurate forecasts are driven by \gls{gpi}. We show that, in order to address \gls{gpi}, global models may require an input window length \textit{strictly larger} than that which would be needed to forecast a single time series~(\Secref{sec:theorem1}). Moreover, higher uncertainty about which specific process is generating the input time series must be mitigated by increasing the window size. We show how, in foundation models, this can induce worse performance-cost trade-offs \wrt domain-specific global models~(\Secref{sec:heterogeneous_vs_window}). 

\item \textbf{Methodology.} We show that decoupling \gls{gpi} and \gls{cf} enables novel designs that can, \eg improve scalability by amortizing the cost of \gls{gpi} through dedicated model components~(\Secref{sec:methodology}). 
\end{itemize}
We support our analysis with theoretical results and empirical evidence on both synthetic and real-world datasets. 
Ultimately, we argue that in time series foundation models, \gls{icl} can be seen as leveraging the input window to reduce uncertainty about which process generated the data.
We believe our findings can inform the design of the next generation of time series foundation models, leading to more efficient and principled approaches that exploit the specifics of the \gls{gpi} and \gls{cf} tasks. 
The paper is organized as follows: \Secref{sec:setting} introduces the proposed problem formulation; \Secref{sec:theory} introduces the \gls{gpi}/\gls{cf} perspective, provides theoretical and empirical evidence in its support, and demonstrates its impact on global and foundation models; \Secref{sec:methodology} assesses a methodology to incorporate our insights in model design.

\section{Problem formulation}\label{sec:setting}

We begin by defining the problem setting and the associated system model. Note that, throughout the paper, with a slight abuse of notation, we use the same symbols to denote both random variables and their realizations.

\paragraph{Problem setting.} Consider a set $\gC$ of $N$ time series,
\[
\gC = \{\vx^i_{0:T^i} \mid i \in \left[1,N\right]\}
\]
where the $i$-th time series $\vx^i_{0:T^i} \coloneq \left(x^i_0, x^i_1, \dots, x^i_{T^i-1}\right)$ consists of a sequence of $T^i$ real-valued univariate observations. The set $\gC$ can be \textit{heterogeneous}, \ie can contain time series from different application domains~(\eg transportation, energy, climatology) and measuring different quantities. We focus on the univariate case for simplicity; the formulation can be extended to the multivariate case by considering architectures, \eg \citep{ansari2025chronos}, that can handle an arbitrary number of input channels. 

\paragraph{Time series forecasting.}
The objective is to learn a parametric model $F(\cdot; \boldsymbol{\Theta})$ that, given a window of observations $\vx^n_{t-W:t}$ ($W\geq 1$), predicts $H$ future values $\vx^n_{t:t+H}$ as
\vspace*{-0.04cm}
\begin{equation}\label{eq:global_predictor}
    \hat{\vx}^n_{t:t+H} = F(\vx^n_{t-W:t}; \, \boldsymbol{\Theta}),
\end{equation}
where $\hat{\vx}^n_{t:t+H}$ denotes predictions, and $\boldsymbol{\Theta}$ are learnable parameters. 
We use the superscript $n$ to denote a generic target time series and cover both the transductive~(${\vx^n_{0:T^n} \in \gC}$) and inductive~(${\vx^n_{0:T^n} \notin \gC}$) scenarios. Note that the latter setting is typical of foundation models. 
Given a family of models, the objective is then to find the parameters' values that minimize the expected point-prediction loss
\vspace*{-0.03cm}
\begin{equation}\label{eq:expected_loss}
    \{\boldsymbol{\Theta}_{opt}\} = \argmin_{\boldsymbol{\Theta}} \gL_W(\boldsymbol{\Theta})
    = \argmin_{\boldsymbol{\Theta}} \mathbb{E}
    \!\left[ \ell\left(\vx^n_{t:t+H}, F\left(\vx^n_{t-W:t}; \boldsymbol{\Theta}\right)\right) \right]
\end{equation}
between the target $\vx^n_{t:t+H}$ and the prediction, as measured by loss function $\ell(\cdot)$, \eg \gls{mse} or \gls{mae}. In the following, we focus on \textit{global} forecasting models, i.e., models that are trained on data from all available $N$ time series, as opposed to \textit{local} models, which are trained on a single time series. Foundation models can be seen as a special case of global models specifically tailored to the inductive setting. 

\paragraph{System model.}
While alternative formulations are possible, we use one that covers a wide range of scenarios. We assume that the generic $n$-th target time series is generated by a process belonging to a family of stochastic processes $\gG$ characterized by a parameter vector $\boldsymbol{\phi}^n$. In particular, we assume the $n$-th time series is generated according to
\vspace*{-0.03cm}
\begin{equation}\label{eq:gen_process}
    x^n_t \sim \sP_x\!\left(x^n_t \vert \vx^n_{<t}; \boldsymbol{\phi}^n\right),
    \quad \text{ where } \quad
    \boldsymbol{\phi}^n \sim \boldsymbol{\gP}_{\boldsymbol{\phi}},
    \quad
    \forall t.
\end{equation}
Here, $\vx^n_{<t}\!\coloneq\!\left(x^n_{t-1}, x^n_{t-2}, \dots\right)$ and $\sP_x\!\left({}\cdot{}; \boldsymbol{\phi}^n\right)$ is a time-invariant parametric distribution. $\boldsymbol{\gP}_{\boldsymbol{\phi}}$ denotes the time-invariant target distribution over possible process parameters $\boldsymbol{\phi}^n$, defined on a $d_p$-manifold $\sM^{d_p}$. 
The processes generating the time series in $\gC$ are then associated with the set of i.i.d.\ random variables $\{\boldsymbol{\phi}^i\}_{i=1}^{N}$ distributed as $\boldsymbol{\gP}_{\boldsymbol{\phi}}$. Note that $\sP_x\!\left(\cdot\right)$ and $\boldsymbol{\gP}_{\boldsymbol{\phi}}$ are shared across all processes.
This system model accounts for the multitude of different dynamics that a global model aims to approximate. For example, we can consider different regions of $\sM^{d_p}$ to correspond to different domains~(\eg meteorology, transportation, utilities). For the data typically used to train foundation models, we can expect $\boldsymbol{\gP}_{\boldsymbol{\phi}}$ to have high variance and multiple modes, \eg one for each target domain. While for domain-specific applications, one could expect $\boldsymbol{\gP}_{\boldsymbol{\phi}}$ to have a smaller variance and result in time series with similar dynamics.

\section{The role of process identification}\label{sec:theory}

In this section, we analyze how a global model can leverage the input observations to produce accurate forecasts across different time series. 
To simplify the notation, we consider $1$-step ahead forecasting.

\subsection{Process identification and conditional forecasting}\label{sec:posterior-average}
Given the task defined in \Secref{sec:setting}, the objective is often to obtain predictions such that
\vspace*{-0.03cm}
\begin{equation}\label{eq:expected_value}
    \hat{\vx}^n_{t} \approx \E\left[x^n_t \given \vx^n_{t-W:t}\right],
\end{equation}
choosing the \gls{mse} for the loss function $\ell\left(\cdot\right)$~(\Eqref{eq:expected_loss})~\citep{gneiting2011making}.
Recall that each time series might be generated by a different stochastic process belonging to the family in \Eqref{eq:gen_process}. 
The expectation in \Eqref{eq:expected_value} can be rewritten using the law of total conditional expectation as
\vspace*{-0.03cm}
\begin{equation}\label{eq:posterior-average}
    \E\!\left[\E\!\left[x^n_t \!\given \vx^n_{t-W:t}, \boldsymbol{\phi}^n\right] \!\given \vx^n_{t-W:t}\right]
    =\!\int_{\sM^{d_p}}\! 
        \underbrace{\E\left[x^n_t \given \vx^n_{t-W:t}, \boldsymbol{\phi}^n\right]}_{\text{\gls{cf}}}
        \underbrace{p\left(\boldsymbol{\phi}^n \given \vx^n_{t-W:t}\right)}_{\text{\gls{gpi}}} d\boldsymbol{\phi}^n.
\end{equation}
\Eqref{eq:posterior-average} shows that optimal predictions can be seen as a \emph{Bayesian posterior average}~\citep{gelman1995bayesian} over the generating processes parameters. This highlights two coupled aspects of the forecasting problem:
\begin{itemize}[leftmargin=1em, itemsep=0cm, topsep=0cm, parsep=0.5em]
    \item \textbf{\Acrfull{gpi}:} the task of inferring the data-generating process, captured by the posterior $p\left(\boldsymbol{\phi}^n \given \vx^n_{t-W:t}\right)$, which quantifies the likelihood of each parameters set $\boldsymbol{\phi}^n$ given the observed sequence $\vx^n_{t-W:t}$.
    \item \textbf{\Acrfull{cf}:} the task of predicting the next observation conditioned on the past, captured by the conditional expectation $\E\left[x^n_t \given \vx^n_{t-W:t}, \boldsymbol{\phi}^n\right]$.
\end{itemize}
Looking at \Eqref{eq:posterior-average} we see how input information relates to \gls{cf} and \gls{gpi}. The entire input sequence would be useful for \gls{gpi}, as it can help contract the posterior $p\left(\boldsymbol{\phi}^n\given\cdot\right)$. Conversely, for the \gls{cf} task, part of the input might be redundant. Consider the case where, given $\boldsymbol{\phi}^n$, $x^n_t$ depends only on the most recent $w < W$ observations, where $W$ is the model's input window length. 
This suggests that \gls{gpi} can increase the number of input observations needed for accurate forecasts beyond that which would be required by \gls{cf} alone.
Moreover, \gls{cf} exploits the temporal dependencies between the input $\vx^n_{t-W:t}$ and the target $x^n_t$, thus the temporal distance between the two is relevant. In contrast, \gls{gpi} relies on the relationship between observations and the process generating them; consequently, ~(\eg if the process is stationary), it can be based on sequences non-adjacent to $x^n_t$. 
We can interpret \gls{gpi} as a form of \gls{icl} from sequences of observations~\citep{li2023transformers}.
Indeed, recent works on language modeling view \gls{icl} as an implicit Bayesian inference algorithm, where the model infers a latent variable characterizing the target task from input/output examples~\citep{xie2022an,zhang2025whatandhow}. 
Under this interpretation, global time series models perform \gls{icl} by leveraging the input time series to infer the target task~(\ie to perform \gls{gpi}). 
This links the window length $W$ to the number of \gls{icl} examples. In standard \gls{icl}, performance improves as this number increases~\citep{agarwal2024manyshot,brown2020language}. Here, longer input windows can improve forecasting accuracy by reducing uncertainty about which process generated the data. 

\begin{figure*}[t]
    \centering
    \includegraphics[width=\linewidth]{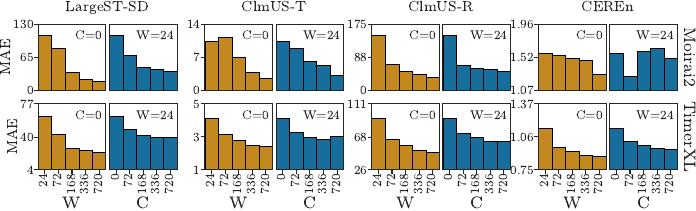}
    \vspace*{-0.25cm}
    \caption{Pre-trained foundation models \gls{mae}, $H\!\!=\!\!24$. For each pair of subplots, the left histogram~(orange) shows the error for increasing \textit{window} length $W$, without \textit{context}, while the right one~(blue) shows the error with the smallest $W$ and increasing \textit{context} length $C$.}
    \label{fig:foundation_context_bars}
    \vspace*{-0.0cm}
\end{figure*}
\paragraph{Evidence in foundation models} 
Under the model in \eqref{eq:gen_process}, \gls{gpi} and \gls{cf} are inherent to the problem of learning global models for groups of time series, and we would expect existing pre-trained foundation models to behave as if implicitly addressing both.
To assess this, we probe two recent foundation models, Moirai2~\citep{liu2025moirai} and TimerXL~\citep{liu2025timer}. Instead of the standard contiguous window of observations, we feed the models a sequence obtained by concatenating two distinct segments: a \textit{window} and a \textit{context}. The window is the standard input covering the observations immediately preceding time-step $t$~(\ie $\vx^n_{t-W:t}$), while the context is a sequence of length $C$ from the same time series, preceding the window and separated from it by a gap of $S$ time-steps~(\ie $\vx^n_{t-W-S-C:t-W-S}$). 
The rationale is that, if these models' performance is mainly due to capturing long-range temporal dependencies for \gls{cf}, using a short window would disrupt their accuracy, and adding context with a temporal gap would not be useful, as the temporal contiguity is broken.
However, if \gls{gpi} is a relevant factor, they would benefit from the additional context. 
We carry out this experiment on four real-world datasets: \gls{largest}~\citep{liu2023largest} (\textit{traffic}), \gls{cer}~\citep{ceren} (\textit{energy}), \gls{climate-t} and \gls{climate-r} (\textit{climate}, see \Appref{app:datasets}). 
\Figref{fig:foundation_context_bars} reports the \gls{mae} when using the standard input $\vx^n_{t-W:t}$ for increasing values of $W$ (blue bars), and compares it against the setup where the input is the concatenation of context and window, \ie $\left[\vx^n_{t'-C:t'} \Vert \vx^n_{t-W:t}\right]$. In this latter case, $W$ is fixed to $W\!=\!24$, and we increment the context length $C$. The context is taken \wrt $t'\!=\!t-W-S$, with $S\!=\!336$, corresponding to a separation of two weeks between context and window. Reasonably, such context should not carry relevant temporal dependencies \wrt the target $\vx^n_{t:t+H}$, while still being representative of the same process that generated the window~(assuming the process does not abruptly change). Across most datasets, the two strategies yield analogous performance trends~(example forecasts are reported in \Appref{app:complement_pretrained_foundation}).
As these models are trained on temporally contiguous inputs, performance gaps are expected; however,  these results suggest that both models use part of the input to infer the target data-generating process, and that -- in the considered datasets -- only a short window of the most recent observations can be enough for effective forecasting. 

\subsection{Process identification and window length}\label{sec:theorem1}
The analysis in \Secref{sec:posterior-average} shows that, in global models, increasing the window length $W$ can lead to more accurate forecasts by reducing uncertainty about which specific process is generating the $n$-th target time series. 
Indeed, \Eqref{eq:posterior-average} shows the role of the input window when forecasting over a group of time series generated by different processes. Optimal predictions can be seen as a weighted average across likely processes; longer observation windows allow for tailoring predictions to the target time series by contracting the posterior $p\left(\boldsymbol{\phi}^n \given \cdot\right)$. Note that this effect does not exist in local models trained on data coming from a single time series.

To illustrate the impact of \gls{gpi} on global forecasting tasks, we consider as a minimal example the case in which the target processes are simple Markovian processes of order $P$, with $P \in \sN^+$. In particular, we consider the sub-family of processes in~\eqref{eq:gen_process} such that the family $\gG$ is determined by
\begin{equation}\label{eq:markov_process}
    x^n_t=g\!\left(\vx^n_{t-P:t}; \boldsymbol{\phi}^n\right) + \epsilon_t^n,
    \quad \text{ where } \quad
    \boldsymbol{\phi}^n \sim \boldsymbol{\gP}_{\boldsymbol{\phi}},
    \text{ and }
    \epsilon_t^n \sim \boldsymbol{\gP}_{\epsilon},
    \quad
    \forall t.
\end{equation}
Here, $g\left({}\cdot{}; \boldsymbol{\phi}^n\right)$ is a well-behaved parametric function, such that each resulting process is stationary. Noise terms $\epsilon_t^n$ are i.i.d.\ , sampled from time-invariant distribution $\boldsymbol{\gP}_{\epsilon}$ with finite variance and zero mean. 
In this setting, $x^n_t$ clearly depends only on the most recent $P$ observations, and a window of length $P$ would capture all relevant temporal dependencies. 
 
\begin{assumption}\label{ass:non_zero_exp}
Assume that data are generated as in \Eqref{eq:markov_process} such that
\[
\E_{\vx^n_{t-P:t}}\!\left[
\Var_{\boldsymbol{\phi}^n}\!\left(\E_{\epsilon_t^n}\!\left[x_t^n \mid \vx^n_{t-P:t},\boldsymbol{\phi}^n\right]
\mid \vx^n_{t-P:t}\right)
\right] > 0.
\]
\end{assumption}
Assumption~\ref{ass:non_zero_exp} states that the expected value of the process variance given the last $P$ observations is greater than zero. Said differently, we assume that processes corresponding to different parameters $\boldsymbol{\phi}^n$ cannot be uniquely identified from the most recent $P$ observations and that, given such observations, there are plausible processes with different conditional expectations. This is a very reasonable assumption, especially when training on related~(noisy) time series.
\begin{tcolorbox}
\begin{theorem}{\textbf{\textup{[Necessity of $W>P$.]}}}\label{prop:necessity_larger_window}
Under the formulation in \Eqref{eq:markov_process}, and Assumption~\ref{ass:non_zero_exp}, let $F\left(\vx^n_{t-W:t}; \boldsymbol{\Theta}_{opt}\right)$ denote an optimal one-step-ahead predictor (in the sense of \Eqref{eq:expected_value}) for input window length $W$. 

\smallskip
Then, a window size $W > P$ is a \textbf{necessary} condition for $F\left(\vx^n_{t-W:t}; \boldsymbol{\Theta}_{opt}\right)$ to achieve the minimum expected error $\Var\left(\epsilon^n_t\right)$.
\end{theorem}
\end{tcolorbox}
The proof of Theorem~\ref{prop:necessity_larger_window} is given in \Appref{proof_2}. 
The takeaway is that, regardless of how far back the true temporal dependencies go (\ie how large $P$ is), \gls{gpi} introduces an overhead, as we might need more than $P$ observations to attain the minimum possible expected error.
Indeed, as shown in \Appref{proof_2}, the optimal expected error of a predictor with window length $W \geq P$ can be written as
\begin{equation}\label{eq:residual-bias}
\E\left[\Var_{\phi^n}\!\left(\E\left[x^n_t \!\given \vx^n_{t-P:t}, \boldsymbol{\phi}^n\right] \!\given \vx^n_{t-W:t} \right)\right]+\Var\left(\epsilon^n_t\right).
\end{equation}
Here, the right term is irreducible error for any window length $W$, \ie the variance of the noise. The left term, instead, reflects the expected error due to the variance in the conditional expectation of $x^n_t$ across parameter values $\boldsymbol{\phi}^n$ that could likely have generated $\vx^n_{t-W:t}$. 
Notably, beyond accounting for temporal dependencies, using longer observation windows can enable more accurate predictions by reducing the first term in~\Eqref{eq:residual-bias}, provided additional observations contract the posterior $p\left(\boldsymbol{\phi}^n \given {}\cdot{}\right)$. 
This suggests that a key motivation for increasing window length in global models is to improve \gls{gpi}, rather than strictly to capture long-range dependencies.
In fact, in practice, achieving accurate forecasts may require an input window size $W$ which substantially exceeds the relevant temporal dependencies of the target data-generating process. 
This is driven by the uncertainty about $\boldsymbol{\phi}^n$ implied by $\boldsymbol{\gP}_{\boldsymbol{\phi}}$ and by how rapidly this uncertainty shrinks as $W$ grows.
However, alongside higher computational costs, increasing the input size often introduces additional challenges, \eg increases model complexity. This trade-off is ever more relevant when considering time series foundation models, which are trained on heterogeneous data coming from a large variety of domains.

\paragraph{Empirical analysis}
\begin{wrapfigure}[14]{r}{0.53\textwidth}
    \centering
    \vspace{-0.5cm}
    \includegraphics[width=\linewidth]{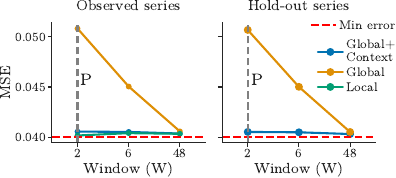}
    \vspace*{-0.25cm}
    \caption{1-step ahead forecasting error (\gls{mse}, 3 runs, $\pm std$) for \gls{nar} processes. $P=2$, $C=46$.}
    \label{fig:controlled_context_exp_single}
    \vspace*{-0.0cm}
\end{wrapfigure}
The need for a window length exceeding the temporal order of the target process can also be investigated empirically through simulations in controlled environments. 
We generate a dataset of $N\!=\!1000$ time series, each sampled from a different \gls{nar} process of order $P\!=\!2$ (see \Appref{app:synthetic_data}).
We train both global and local models (implemented as a single-layer \gls{gru}~\citep{cho2014properties}) for input windows of varying length $W$. The global model's parameters are shared across all time series, while the local model is trained from scratch on each time series. \Figref{fig:controlled_context_exp_single} reports the \gls{mse} achieved in the transductive~(left) and inductive~(right) settings (see \Appref{app:data_splitting} for details on dataset splits). Note that local models are inherently transductive-only.
As one would expect, the local model~(green curve) achieves near-optimal error (red line) already at $W\!\!=\!P\!\!=\!2$, while the global model~(orange curve) approaches it only for $W \!\gg\! P$. As each generating process has order $P\!=\!2$ by construction, we can explain the additional observations needed by the global model as the effect of uncertainty about the target data-generating process. 

In the next step, we train a variant of the global \gls{gru} model~(details in \Appref{app:gru_with_context}) which, similar to the experiment in \Figref{fig:foundation_context_bars}, takes as input a sequence obtained by concatenating a past \textit{context} $\vx^n_{t'-C:t'}$ of length $C$ and a recent \textit{window} $\vx^n_{t-W:t}$ of length $W$. During training the context's location $t'$ is drawn uniformly from $\left[t\!-\!W\!-\!S, t\!-\!W\right)$ to avoid bias towards specific temporal lags, while at test time it is fixed to $t\!-\!W\!-\!S$.
From the analysis in \Secref{sec:posterior-average}, we expect that \gls{gpi} can be addressed by leveraging any sequence generated by the target process, not only those adjacent to the target $x^n_t$. In fact, under the formulation in \Eqref{eq:markov_process}, $p\!\left(\boldsymbol{\phi}^n\given\vx^n_{t-W:t}\right)$ is invariant to time shifts. 
The blue curve in \Figref{fig:controlled_context_exp_single} corresponds to the results of this experiment, using $S\!=\!50$ as the maximum temporal gap between window and context. Notably, adding the context enables near-minimum error already at $W\!=\!P$, supporting the idea that, even if here additional observations beyond the last $P$ time lags are not relevant for their temporal relationship with $x^n_t$, they indeed serve to reduce uncertainty about the target data-generating process (additional results in \Appref{app:complement_empirical_evidence_theory}, including experiments for $P\!>\!2$). 

\subsection{The cost of \gls{gpi} in foundation models}\label{sec:heterogeneous_vs_window}
Consider the typical use case for a foundation model. The model is pretrained over multiple domains, then typically used \emph{zero-shot} by different users on a specific target domain. Although this flexibility is clearly an advantage, according to our analysis in \Secref{sec:theorem1}, it might entail a cost in terms of the required input window length, compared to a domain-specific global model.
Indeed, we expect the target data distribution for a single end user to have much narrower support and lower variance than a foundation model's pretraining distribution. However, in the absence of a mechanism to efficiently condition the foundation model on the desired target distribution, the input sequence length required to resolve uncertainty about the target process will still depend on the pretraining distribution. 
Indeed, by looking at published results, foundation models~\citep{woo2024unified, liu2025timer} require longer observation windows than global models trained directly on the target domain~\citep{nie2022time, zeng2023transformers, chen2023tsmixer, luo2024moderntcn}. The following experiments highlight this phenomenon associated with global models. 

\paragraph{Empirical analysis: controlled environments} To investigate how training on multiple domains impacts global models, we design the following controlled experiment. 
We simulate different domains by making each domain correspond to \gls{nar} processes of a specific order $P \in \{1,2,4,6\}$. For each domain, we generate a dataset of $N=500$ time series, where the parameters of the process generating each series are randomly sampled from a uniform distribution. We then train a single-layer \gls{gru}, from scratch, either on a single dataset (single-domain) or jointly on multiple datasets (multi-domain).\footnote{Both single- and multi-domain training involve multiple data-generating processes, so both settings would involve \gls{gpi}.} Finally, each model is separately evaluated on a hold-out set of unseen time series for each of the domains they where trained on. Results are summarized in \Figref{fig:multitask_controlled_exp}. Each subplot corresponds to the test performance on a specific domain (\ie processes of a specific order). Colors denote different training data distributions (either single-domain or multi-domain). Points on each curve correspond to varying the input window length hyperparameter $W$ (additional results in \Appref{app:complement_heterogeneous_vs_window}). 
The results show that, when considering performance on a specific domain for a fixed $W$, training on multiple domains consistently leads to worse performance \wrt training directly on the target domain. We attribute this to increased uncertainty about the specific process generating the observed input sequence when the training data include multiple domains. 
At the same time, increasing $W$ allows models trained jointly on multiple domains to match, or surpass, the performance of models trained on a single domain. 
We attribute this effect to a reduction of the uncertainty about the target data-generating process given longer input sequences. 

\begin{figure*}[t]
    \centering
    \includegraphics[width=\linewidth]{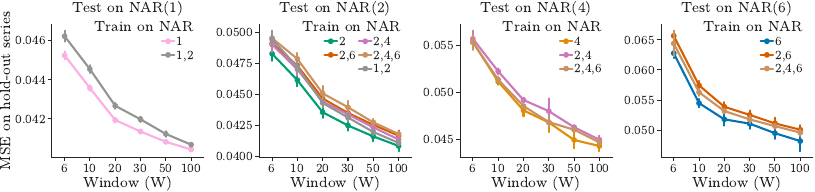}
    \vspace*{-0.5cm}
    \caption{1-step ahead forecasting (\gls{mse}, \textit{inductive}, 3 runs $\pm std$) for different combinations of \gls{nar} domains. Columns show error on a specific domain, color identifies the training domains.}
    \label{fig:multitask_controlled_exp}
    \vspace*{-0.0cm}
\end{figure*}
\begin{figure*}[t]
    \centering 
    \includegraphics[width=\linewidth]{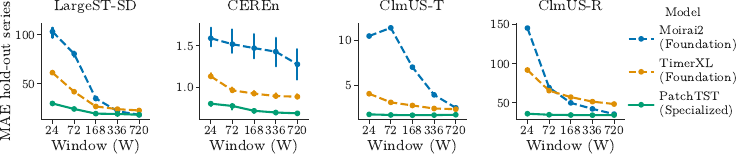}
    \vspace*{-0.3cm}
    \caption{24-step ahead forecasting error (\gls{mae}, \emph{inductive}, 3 runs $\pm std$) against window length $W$. Foundation models (Foundation, \emph{dashed}) against domain-specific global models (Specialized, \emph{solid}).}
    \label{fig:multitask_real_exp}
    \vspace*{-0.0cm}
\end{figure*}
\paragraph{Empirical analysis: real-world} We extend the analysis to real-world scenarios by considering the pretrained foundation models and datasets from \Secref{sec:posterior-average}. In particular, we compare the foundation models to domain-specific global models in \emph{inductive} settings. We consider PatchTST~\citep{nie2022time} as a representative architecture for the domain-specific models (details in \Appref{app:sota_architectures}). For each dataset, and each window lengths $W$, we train it from scratch (details in \Appref{app:training_details}) using only time series from the target dataset. Finally, for each dataset and window length $W$, we compare PatchTST's performance to that of the foundation models on time series entirely held out during training.
\Figref{fig:multitask_real_exp} reports the experiment's results. Subplots denote the dataset, line color denotes the model, and line-style denotes whether the model is foundation (dashed) or domain-specific (solid). Clearly, pretraining on a very heterogeneous data distribution results in the foundation models requiring a much larger input window to match the performance of the domain-specific models. Note that the foundation models' pretraining corpora contain time series from the considered domains.

These results confirm that the flexibility of foundation models comes at a cost, a cost clearly linked to \gls{gpi}. Additional results in \Appref{app:linear_on_linear} confirm that \gls{gpi} might require higher model capacity than \gls{cf}. In resource-constrained environments, considering these aspects is relevant and can make a domain-specific alternative preferable to a foundation model. 

\section{Leveraging insights for model design}\label{sec:methodology}
We now show how our analysis can inform model design and provide practical advantages in real-world settings.
For this purpose, we demonstrate a methodology to amortize the inference cost associated to \gls{gpi}.
Recall insights from \Secref{sec:posterior-average}: (i) \gls{gpi} can require larger observation windows than \gls{cf}; (ii) \gls{gpi} may rely on sequences generated by the target process, but not adjacent to the target $x^n_t$.
Therefore, we can try to improve model scalability by separating  \gls{gpi} and \gls{cf}.

\paragraph{Model design}
We aim at~(partially) decoupling \gls{cf} and \gls{gpi} by using predictors of the form:
\begin{subequations}\label{eq:explicit_predictor}
\begin{align}
    \ve^n &= G\!\left(\vx^n_{t'-C:t'};\,\boldsymbol{\Theta}_G\right)  \label{eq:explicit_identifier}\\
    \hat{\vx}^n_{t:t+H} &= F\!\left(\vx^n_{t-W:t},\,\ve^n;\,\boldsymbol{\Theta}_F\right) \label{eq:explicit_forecaster}
\end{align}
\end{subequations}
where $t'<t-W$ and $C > W$. Here, $G\left(\cdot\right)$ takes as input a \textit{context} $\vx^n_{t'-C,t'}$ drawn from the target time series history, and outputs a latent representation $\ve^n$; $F\left(\cdot\right)$, instead, operates on the \emph{window} $\vx^n_{t-W,t}$ containing the latest observations, and it is further conditioned on $\ve^n$.
The rationale is that $G\left(\cdot\right)$ can produce an embedding $\ve^n$, summarizing information about the process generating the $n$-th time series. This can be cached and re-used for subsequent forecasts of the same time series, reducing the inference cost compared to the standard contiguous window approach using $C\!+\!W$ observations. 
In fact, at inference time, predicting a target time series for $m$ steps would require processing ${C \!+\! m \cdot W}$ observations, instead of ${m \cdot \left(C \!+\! W\right)}$. In applications where \gls{gpi} can require much more information than \gls{cf}, \eg foundation models, we can set $C \!\gg\! W$ to increase the number of observations for \gls{gpi}, with no impact on the cost of repeated forecasts for the same time series. Similarly, if one expects \gls{gpi} to require more capacity than \gls{cf}, the architecture of $G(\cdot)$ can be more complex than that of $F(\cdot)$. Note that, if the data-generating process drifts over time, one could regularly recompute $\ve^n$ from a more recent context.

Models of the form in \Eqref{eq:explicit_predictor} can be advantageous \wrt popular architectures for global models. For non-recurrent models (\eg \acrlongpl{mlp}, \textit{encoder-only} Transformers) the computational advantage is clear, as they need to process the entire input sequence for each forecast. Most recurrent global models, instead, are based on input-patching (\eg \textit{decoder-only} Transformers), where the input window is split into multiple patches of length $L$. In this case, even if the architecture is recurrent, for a new forecast, one either needs to wait and collect $L$ new observations, or the entire input sequence must be processed from scratch. This can increase the inference cost for end users of foundation models, as these models are mostly patch-based Transformers~\citep{liang2024foundation}, but the patch-size $L$ cannot be tuned to the needs of any specific application. The formulation in \Eqref{eq:explicit_predictor} can help manage cost-to-performance trade-offs arising from \gls{gpi}, regardless of the implementation of $G\left(\cdot\right)$ and $F\left(\cdot\right)$ (further discussion in \Appref{app:foundation_caching}). 
Finally, this decoupled design naturally accommodates intermittent data streams, where long contiguous observation windows might not be available.

\begin{figure*}[t]
    \centering
    \includegraphics[width=\linewidth]{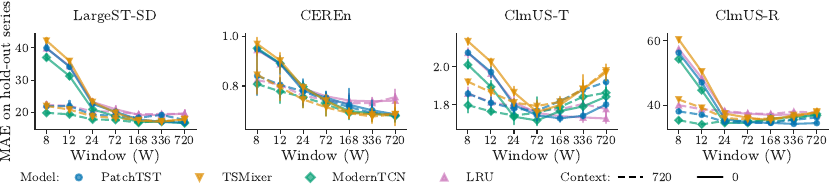}
    \vspace*{-0.25cm}
    \caption{24-step ahead forecasting error (\gls{mae}, \textit{inductive}, 3 runs $\pm std$). Color/marker denotes the model. Line-style denotes the approach: decoupling (dashed) or standard (solid).}
    \label{fig:real_world_context_exp}
    \vspace*{-0.0cm}
\end{figure*}
\paragraph{Empirical analysis}
We quantify the performance-to-cost trade-offs through evaluation on the real-world datasets introduced in \Secref{sec:posterior-average}. We implement the \textit{base model} in \Eqref{eq:explicit_forecaster} with representative \acrlong{sota} architectures: PatchTST~\citep{nie2022time} (attention), TSMixer~\citep{chen2023tsmixer} (fully connected), ModernTCN~\citep{luo2024moderntcn} (convolutional), and \gls{lru}~\citep{orvieto2023resurrecting} (recurrent); details in \Appref{app:sota_architectures}. 
The \textit{embedding module}~(\Eqref{eq:explicit_identifier}) is always implemented by PatchTST for consistency. 
Depending on the implementation of the base model, $\ve^n$ is integrated via summation or concatenation~(see \Appref{app:sota_architectures}). For each base model, hyperparameters are the same for both the standard and decoupled version. 
When training decoupled models~(\Eqref{eq:explicit_predictor}), $\vx^n_{t'-C:t'}$ is drawn uniformly from $\left[t-W-S,t-W\right)$, where $S$ specifies a maximum back-shift to prevent sampling outdated contexts. Randomizing the temporal distance between $t'$ and $t$ also prevents bias toward specific time lags. During testing, we draw the context using the maximum back-shift considered during training. Details on the experimental setting can be found in \Appref{app:training_details}.

Fig.~\ref{fig:real_world_context_exp} reports the \gls{mae} obtained when training each model to forecast the next $H\!=\!24$ time-steps. Points on the curve denote results from training with different window lengths $W$, using context length $C=720$ and maximum back-shift $S=336$. Line color denotes the base model. Solid lines denote the performance of the standard \gls{sota} architectures~(implementing \Eqref{eq:global_predictor}), while dashed lines are their counterparts with input decoupling~(\Eqref{eq:explicit_predictor}). Across the considered datasets and models, the \gls{mae} of the decoupled approach at small values of $W$, is often comparable to that achieved by the base model at larger $W$ values. For higher $W$ values, the \gls{mae} for the two strategies is similar. In such a case, $\vx^n_{t-W:t}$ might be already sufficient to specialize the model's predictions. 

\begin{wrapfigure}[23]{r}{0.52\textwidth}
    \vspace*{-0.0cm}
    \centering
    \includegraphics[width=\linewidth]{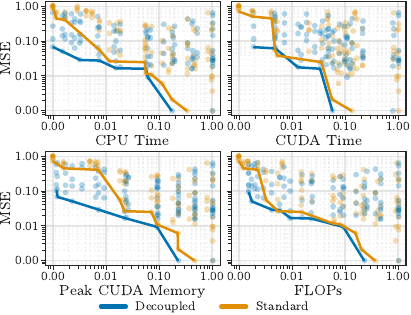}
    \vspace*{-0.25cm}
    \caption{Pareto frontiers for the decoupled (blue) and standard (orange) approach. Dots correspond to model/dataset pairs from \Figref{fig:real_world_context_exp} (3 runs average). \gls{mse} (y-axis, \textit{inductive}) is min-max normalized per-dataset, while cost metrics (x-axis) per-architecture.}
    \label{fig:pareto_merged}
    \vspace*{-0.0cm}
\end{wrapfigure} 
In \Figref{fig:pareto_merged} we report performance-cost Pareto curves comparing the standard (orange curve) and decoupled (blue curve) approach. We consider models from the previous experiment, and plot their \gls{mse} against inference speed, \gls{flops}, and memory occupancy, considering the case where $\ve^n$ is kept fixed across multiple forecasts. 
Each point in the plots corresponds to a specific architecture, dataset, window size and implementation design (either standard or decoupled). The \gls{mse} is normalized per-dataset, while cost metrics are normalized per-architecture (see \Appref{app:complement_empirical_evidence_design} for detailed per model/dataset pair results). 
We can observe how the decoupled approach is the Pareto optimal one.
Overall, these results support the practical applicability of our insights, suggesting that model scalability can be improved by considering \gls{gpi} and \gls{cf} during model design.

\paragraph{Discussion}
Ultimately, our results advocate for a change in the design perspective for time series foundation models. Instead of relying on large contiguous input windows and monolithic architectures, one could design architectures decoupling \gls{gpi} and \gls{cf}. For example, \gls{gpi} could benefit from metadata about the target time series. Lightweight downstream predictors could be tailored via reusable embeddings extracted from rich contexts. Considering \gls{gpi} and \gls{cf} distinctly can broaden the design space for time series models towards more flexible and efficient designs.

\section{Related Works}\label{sec:related-works}

The implications of training a single model on large collections of time series, rather than having models specific to each time series, have often been overlooked. 
\citet{montero2021principles} note how transductive global models need longer input windows to match the performance of local approaches. We expand their findings, leveraging \gls{gpi} to provide an explanation of the mechanisms by which increasing the window size enables better performance also in inductive scenarios (\eg foundation models~\citep{kottapalli2025foundation}). 
Recent works started exploring \gls{icl}~\citep{dong2024survey} for time series forecasting, with approaches mostly relying on exogenous information.
\citet{faw2025context} use related time series signals, \citet{auer2025zero,ansari2025chronos} adopt covariate signals, while \citet{williams2025context} feed the model textual descriptors. Notably, \citet{lu2025incontext} instead augment the model input with example past/future pairs from the target time series itself. 
In our framework, the effectiveness of these approaches can be explained with the role played by \gls{gpi}. 

Finally, in transductive learning settings, different works on \textit{global-local} models~\citep{smyl2020hybrid,cini2023taming,butera2025on} show how global models can be augmented with local components to efficiently address \gls{gpi} at training time. In this work, we address inductive settings where one cannot afford local components, and discuss how purely global models can instead leverage the input sequence to address \gls{gpi} at inference time.
See \Appref{app:related_works} for further discussion of related works.

\section{Conclusion}\label{sec:conclusions}

In this paper, we analyze key aspects of deep learning for time series forecasting, showing that effective prediction hinges on the interplay between \gls{gpi} and \gls{cf}.
Our theoretical and empirical analyses clarify that long observation windows can improve performance by improving \gls{gpi}, and show that model design choices encouraging the decoupling of the mentioned tasks can improve scalability.
Altogether, our insights can inform the next generation of time series foundation models.

\paragraph{Limitations and Future works} 
While we demonstrate the practical relevance of our insights, we leave the development and evaluation of a foundation model based on said insights as future work. 
In this regard, future research should consider more sophisticated approaches, \eg based on dynamically selecting the number of recent observations to process to balance accuracy and efficiency in \gls{cf}. 
Moreover, investigating methodologies to decouple \gls{gpi} in the context of highly time-varying data-generating processes is another interesting direction for future studies.

\bibliography{references}


\appendix

\section{Terminology}\label{app:terminology}
We distinguish here several related terms. 
\begin{itemize}
    \item \textbf{\Acrlong{gpi}} refers to inferring which latent stochastic process generated an observed sequence, without necessarily recovering parameters. 
    \item \textbf{System identification} usually denotes estimating explicit parametric models of dynamical systems, often with control in mind. 
    \item \textbf{Dynamics identification} emphasizes learning the functional form governing state transitions. 
    \item \textbf{Model identification} refers more broadly to choosing or estimating a model structure consistent with observed data. 
    \item \textbf{Process characterization} targets statistical properties of the series (e.g., stationarity, ergodicity, dependence structure) rather than the underlying dynamics.  
    \item \textbf{In-context learning} denotes the ability of a deep learning model to adapt to a new task on the fly, without parameter updates, by conditioning on examples of the target task. 
\end{itemize}
Throughout this work, we use \acrfull{gpi} in this narrow sense of reducing the uncertainty about the underlying generating process given a finite context window.

\section{Proofs}\label{app:proofs}

\subsection{Proof of Theorem~\ref{prop:necessity_larger_window}}\label{proof_2}

\begin{tcolorbox}
\textbf{\textup{[Necessity of $W>P$.]}} 
Under the formulation in \Eqref{eq:markov_process}, and Assumption~\ref{ass:non_zero_exp}, let $F\left(\vx^n_{t-W:t}; \boldsymbol{\Theta}_{opt}\right)$ denote an optimal one-step-ahead predictor (in the sense of \Eqref{eq:expected_value}) for input window length $W$. 

\smallskip
Then, a window size $W > P$ is a \textbf{necessary} condition for $F\left(\vx^n_{t-W:t}; \boldsymbol{\Theta}_{opt}\right)$ to achieve the minimum expected error $\Var\left(\epsilon^n_t\right)$.
\end{tcolorbox}

\begin{proof}
We proceed by contradiction. Assume that there exists a predictor $F\left(\cdot; \boldsymbol{\Theta}_{opt}\right)$ that achieves the minimum possible expected error on every process in $\gG$ (\Eqref{eq:markov_process}). We will show that under the $W \leq P$ assumption, such predictor does not exist.

We consider the two cases:

\textbf{Case $W < P$}. This case is immediate. Even for a fixed process, an order-$P$ Markov model may require the last $P$ observations to achieve the minimum possible expected error. Without additional assumptions guaranteeing that the conditional mean is a function of only $W$ lags, no predictor using only $\vx^n_{t-W:t}$ can, in general, recover the conditional expectation.

\textbf{Case $W = P$.}
Consider any predictor $F(\vx^n_{t-P:t};\boldsymbol{\Theta})$ that observes only the last $P$ values.
Its expected \gls{mse} under the data-generating mechanism induced by
$\boldsymbol{\phi}^n\sim \boldsymbol{\gP}_{\boldsymbol{\phi}}$ and \Eqref{eq:markov_process} is
\begin{equation}\label{eq:risk_def_P}
\mathcal R(\boldsymbol{\Theta})
\;\coloneq\;
\E
\!\left[\bigl(x_t^n - F(\vx^n_{t-P:t};\boldsymbol{\Theta})\bigr)^2\right].
\end{equation}

By \eqref{eq:expected_value}, an \gls{mse}-optimal predictor satisfies
\begin{equation}\label{eq:opt_pred_P}
F(\vx^n_{t-P:t};\boldsymbol{\Theta}_{opt})
\simeq
\E\!\left[x_t^n \mid \vx^n_{t-P:t}\right].
\end{equation}
Hence, the minimum attainable expected error among predictors is 
\begin{align}\label{eq:bayes_risk_P}
\mathcal R(\boldsymbol{\Theta}_{opt})
= \, &
\E
\!\left[\bigl(x_t^n - \E\!\left[x_t^n \mid \vx^n_{t-P:t}\right]\bigr)^2\right] \\
= \, &
\E_{\vx^n_{t-P:t}}
\!\left[
\Var\!\left(x_t^n \mid \vx^n_{t-P:t}\right)
\right].
\end{align}

We now expand the conditional variance by conditioning on the latent parameter $\boldsymbol{\phi}^n$.
By the law of total variance,
\begin{equation}\label{eq:totvar_phi_P}
\Var\!\left(x_t^n \mid \vx^n_{t-P:t}\right)
 =  
\E_{\boldsymbol{\phi}^n}\!\left[
\Var_{\epsilon_t^n}\!\left(x_t^n \mid \vx^n_{t-P:t},\boldsymbol{\phi}^n\right)
\;\middle|\; \vx^n_{t-P:t}
\right] 
\, + \,
\Var_{\boldsymbol{\phi}^n}\!\left(
\E_{\epsilon_t^n}\!\left[x_t^n \mid \vx^n_{t-P:t},\boldsymbol{\phi}^n\right]
\;\middle|\; \vx^n_{t-P:t}
\right).
\end{equation}

Under the signal-plus-noise model in \Eqref{eq:markov_process}, with $\epsilon_t^n$ independent of $\boldsymbol{\phi}^n$ and $\vx^n_{<t}$ (hence of $\vx^n_{t-P:t}$), and with time-invariant variance, the first term reduces to
\[
\Var_{\epsilon_t^n}\!\left(x_t^n \mid \vx^n_{t-P:t},\boldsymbol{\phi}^n\right)
=
\Var_{\epsilon_t^n}\!\left(\epsilon_t^n \mid \vx^n_{t-P:t},\boldsymbol{\phi}^n\right)
=
\Var(\epsilon_t^n)
\]
Substituting into \eqref{eq:totvar_phi_P} yields
\begin{equation}\label{eq:cond_var_decomp}
\Var\!\left(x_t^n \mid \vx^n_{t-P:t}\right)
=
\Var(\epsilon_t^n)
+
\Var_{\boldsymbol{\phi}^n}\!\left(\E_{\epsilon_t^n}\!\left[x_t^n \mid \vx^n_{t-P:t},\boldsymbol{\phi}^n\right]
\mid \vx^n_{t-P:t}\right).
\end{equation}
Taking $\E_{\vx^n_{t-P:t}}[\cdot]$, this gives
\begin{equation}\label{eq:risk_decomp}
\mathcal R(\boldsymbol{\Theta}_{opt})
=
\Var(\epsilon_t^n)
+
\E_{\vx^n_{t-P:t}}\!\left[
\Var_{\boldsymbol{\phi}^n}\!\left(\E_{\epsilon_t^n}\!\left[x_t^n \mid \vx^n_{t-P:t},\boldsymbol{\phi}^n\right]
\mid \vx^n_{t-P:t}\right)
\right].
\end{equation}

However, by the theorem's assumptions
\[
\E_{\vx^n_{t-P:t}}\!\left[
\Var_{\boldsymbol{\phi}^n}\!\left(\E_{\epsilon_t^n}\!\left[x_t^n \mid \vx^n_{t-P:t},\boldsymbol{\phi}^n\right]
\mid \vx^n_{t-P:t}\right)
\right]) > 0.
\]
Therefore,
\[
\mathcal R(\boldsymbol{\Theta}_{opt}) > \Var_{\epsilon_t^n}(\epsilon_t^n),
\]
so no predictor with $W=P$ can attain the minimum expected error
$\Var_{\epsilon_t^n}(\epsilon_t^n)$ for every target process.

\qedhere

\end{proof}

\section{Example with \acrshort{ar}(2)}\label{app:ar}

\begin{example}[two \acrshort{ar}(2) processes]\label{prop:single_ts_case_linear}
Consider the case where the system model in \Secref{sec:setting} is fully specified and defined as
\begin{equation}\label{eq:gen_process_ar2}
    \begin{cases}
        \, x^n_t = \phi^n_1 \cdot x^n_{t-1} + \phi^n_2 \cdot x^n_{t-2} + \epsilon^n_t \\[4pt]
        \, \epsilon^n_t \sim \gN\left(0,\sigma^2\right), \quad \boldsymbol{\phi}^n \sim Categorical\left(\left(\boldsymbol{\phi}^{n_1}, 0.5\right),\left(\boldsymbol{\phi}^{n_2}, 0.5\right)\right) \\[4pt] 
    \end{cases}
    ,\forall t,
\end{equation}
where $\boldsymbol{\phi}^n = \left(\phi^n_1, \phi^n_2\right)$. In this case, processes in $\gG$ are \gls{ar} of order $P=2$; the noise follows a zero-mean Gaussian distribution with the same variance for all possible processes, and $\boldsymbol{\gP}_{\boldsymbol{\phi}}$ is a categorical distribution assigning uniform probability to two possible parameter values $\boldsymbol{\phi}^{n_1}$ and $\boldsymbol{\phi}^{n_2}$, which we assume to correspond to stationary \gls{ar} processes.

\paragraph{$\mathbf{W>P}$ is necessary to achieve the minimum expected error for every target process:}
Let $F\left(x^n_{t-1}, x^n_{t-2}; \boldsymbol{\Theta}_{opt}\right)$ denote an optimal predictor (in
the sense of \Eqref{eq:expected_value}) for input window length $W=2$, trained via \gls{mse} minimization. 

Then, for every input window $\left(x^n_{t-1}, x^n_{t-2}\right)$, the predictor approximates the conditional expectation $\mathbb{E}_{\epsilon^n_t}\left[x^n_t \given x^n_{t-1}, x^n_{t-2}\right]$. 
By the law of total expectation, this conditional expectation can be extended by further conditioning on the latent categorical random variable $\boldsymbol{\phi}^n$.
\begin{align}
    F\left(x^n_{t-1}, x^n_{t-2}; \boldsymbol{\Theta}_{opt}\right) & \simeq  \mathbb{E}_{\epsilon^n_t}\left[x^n_t \given x^n_{t-1}, x^n_{t-2}\right] \\
    & = \mathbb{E}_{\boldsymbol{\phi}^n}\left[\E_{\epsilon^n_t}\left[x^n_t \given x^n_{t-1}, x^n_{t-2}, \boldsymbol{\phi}^n\right] \given x^n_{t-1}, x^n_{t-2} \right] \\
\end{align}
Let us define:
\begin{equation}
    \pi := p\left(\boldsymbol{\phi}^n=\boldsymbol{\phi}^{n_1} \given x^n_{t-1}, x^n_{t-2}\right) \label{eq:process_likelihood_ar2}
\end{equation}
as the posterior probability of the process parameters taking value $\boldsymbol{\phi}^{n_1}$.
The outer expectation can be expanded as follows:
\begin{align}
    F\left(x^n_{t-1}, x^n_{t-2}; \boldsymbol{\Theta}_{opt}\right) & \simeq \pi \cdot \E_{\epsilon^n_t}\left[x^n_t \given x^n_{t-1}, x^n_{t-2}, \boldsymbol{\phi}^n=\boldsymbol{\phi}^{n_1}\right] + \left(1-\pi\right) \cdot \E_{\epsilon^n_t}\left[x^n_t \given x^n_{t-1}, x^n_{t-2}, \boldsymbol{\phi}^n=\boldsymbol{\phi}^{n_2}\right] \label{eq:expansion_ar2} \\ 
    & = \pi \cdot \left(\phi^{n_1}_1 x^n_{t-1} + \phi^{n_1}_2 x^n_{t-2}\right) + \left(1-\pi\right)\cdot\left(\phi^{n_2}_1 x^n_{t-1} + \phi^{n_2}_2 x^n_{t-2}\right) \\
    & = \underbrace{\left( \pi \phi^{n_1}_1 + \left(1-\pi\right) \phi^{n_2}_1 \right)}_{\bar\phi_1} x^n_{t-1} + \underbrace{\left( \pi \phi^{n_1}_2 + \left(1-\pi\right) \phi^{n_2}_2 \right)}_{\bar\phi_2} x^n_{t-2} \label{eq:convex_comb_ar2}
\end{align}
\Eqref{eq:convex_comb_ar2} shows that an optimal predictor behaves as a convex combination of the two possible \acrshort{ar} parameter values, where the weights of the combination are dynamically adapted to the input window depending on $\pi$. 
Considering \Eqref{eq:risk_decomp} (see \Appref{proof_2}), for any target process following \Eqref{eq:gen_process_ar2}, the optimal predictor can attain the minimum possible expected error $\Var_{\epsilon^n_t}\left(\epsilon^n_t\right)$ if and only if
\begin{equation*}
    \E_{x^n_{t-1},x^n_{t-2}}\!\left[\Var_{\boldsymbol{\phi}^n}\!\left(\E_{\epsilon_t^n}\!\left[x_t^n \given x^n_{t-1},x^n_{t-2},\boldsymbol{\phi}^n\right]
\mid x^n_{t-1},x^n_{t-2}\right)
\right] = 0
\end{equation*}

This condition is only satisfied if, for all possible $\left(x^n_{t-1}, x^n_{t-2}\right)$, one of the following edge cases is met:
\begin{itemize}
    \item $\pi = 1$ and $\boldsymbol{\phi}^n=\boldsymbol{\phi}^{n_1}$ (or $\pi = 0$ and $\boldsymbol{\phi}^n=\boldsymbol{\phi}^{n_2}$)
    \item $\phi^{n_1}_1 \cdot x^n_{t-1} + \phi^{n_1}_2 \cdot x^n_{t-2} = \phi^{n_2}_1 \cdot x^n_{t-1} + \phi^{n_2}_2 \cdot x^n_{t-2}$
\end{itemize}
In the first case, input observations unequivocally determine the parameters that generated them, while, in the second case, the two possible values $\boldsymbol{\phi}^n$ can take imply the same conditional expectation for $x^n_t$. Notably, this latter case includes the trivial scenario where $\boldsymbol{\phi}^{n_1} = \boldsymbol{\phi}^{n_2}$ ($\boldsymbol{\gP}_{\boldsymbol{\phi}}$ is degenerate).
Otherwise, under input length $W=P=2$, an optimal predictor cannot attain the minimum expected error. In fact, there would exist cases where
\begin{align*}
    \boldsymbol{\phi}^n=\boldsymbol{\phi}^{n_1} \quad &\text{and} \quad F\left(x^n_{t-1}, x^n_{t-2};\boldsymbol{\Theta}_{opt}\right) \neq \phi^{n_1}_1 x^n_{t-1} + \phi^{n_1}_2 x^n_{t-2} \quad \text{or} \\ \boldsymbol{\phi}^n=\boldsymbol{\phi}^{n_2} \quad &\text{and} \quad F\left(x^n_{t-1}, x^n_{t-2};\boldsymbol{\Theta}_{opt}\right) \neq \phi^{n_2}_1 x^n_{t-1} + \phi^{n_2}_2 x^n_{t-2}.
\end{align*}

\begin{figure}[t]
    \centering
    \includegraphics[width=\linewidth]{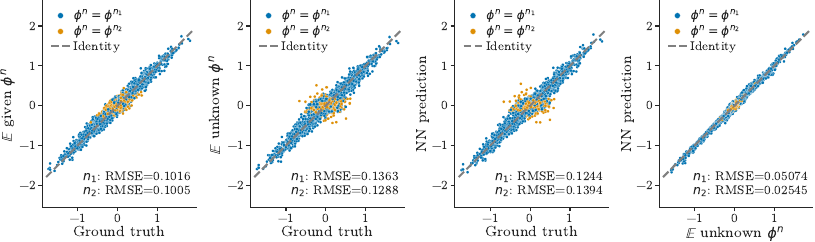}
    \caption{Parity plots for Example~\ref{prop:single_ts_case_linear}. Data generation hyper-parameters: $\sigma^2=0.1$, $seed=7$, $\#timesteps=10000$. \textit{Ground truth} denotes the actual realizations from the stochastic process. $\E$ \textit{given} $\boldsymbol{\phi}^n$ denotes the expected value for the future observation when knowing the process parameters and the last $2$ observations. $\E$ \textit{unknown} $\boldsymbol{\phi}^n$ denotes the predictions obtained by using the real process parameters weighted by the likelihood in \Eqref{eq:process_likelihood_ar2} computed in closed form. \textit{NN prediction} denotes the predictions obtained by a neural network optimized to minimize the \gls{mse}. The \gls{rmse} is computed between the two variables on each plot's axes, and reported separately for processes with $\boldsymbol{\phi}^n=\boldsymbol{\phi}^{n_1}$ and $\boldsymbol{\phi}^n=\boldsymbol{\phi}^{n_2}$.}
    \label{fig:parity_theory_ar}
\end{figure}
To show the alignment between the theory and the empirical case we generate a synthetic dataset by randomly selecting $\boldsymbol{\phi}^{n_1}$ and $\boldsymbol{\phi}^{n_2}$, and train a neural network to minimize the \gls{mse} for the $1$-step ahead prediction.
Specifically, we use a network of the following form:
\begin{align}
    \left[p^{n_1}, p^{n_2}\right] &= softmax\left(\mW_2 \sigma\left(\mW_1 \vx_{t-2:t} + b_1\right) + b_2\right),\\
    \hat{x}^n_t &= \left( p^{n_1} \psi^{n_1}_1 + p^{n_2} \psi^{n_2}_1 \right) x^n_{t-1} + \left(p^{n_1} \psi^{n_1}_2 + p^{n_2} \psi^{n_2}_2\right) x^n_{t-2},
\end{align}
with learnable parameters $\mW_1 \in \sR^{4 \times 2}$, $\mW_2 \in \sR^{2 \times 4}$, $b_1$, $b_2$, $\psi^{n_1}_1$, $\psi^{n_1}_2$, $\psi^{n_2}_1$ and $\psi^{n_2}_2 \in \sR^{1}$. Here $\sigma$ denotes the \textit{Rectified Linear Unit} activation function.

Fig.~\ref{fig:parity_theory_ar} compares predictions obtained with different methods on an hold-out set of test inputs $\left(x^n_{t-1},x^n_{t-2}\right)$. 
The rightmost plot compares predictions produced by the learned neural network with those obtained by computing \Eqref{eq:convex_comb_ar2} in closed form, showing a clear alignment.
The first three plots from the left compare the ground truth $x^n_t$ to the prediction of three different strategies: neural network (third from the left),  closed form \Eqref{eq:convex_comb_ar2} (second from the left) and computing $\E_{\epsilon^n_t}\left[x^n_t \given x^n_{t-1}, x^n_{t-2}, \boldsymbol{\phi}^n\right]$ given knowledge about the value of $\boldsymbol{\phi}^n$ for the process that generated the input (first from the left).
We can see how the first two strategies are worse (higher spread around the identity line) than computing the expected value given knowledge about $\boldsymbol{\phi}^n$.
\Appref{app:complement_theory_ar} shows additional results for different synthetic dataset seeds.
These results highlight how, in this example case: (i) a neural model can behave as \Eqref{eq:convex_comb_ar2}; (ii) an input window of length $W=P=2$ is not sufficient to achieve the minimum attainable error on both possible processes at once.

\end{example}

\section{Process identification and model capacity} \label{app:linear_on_linear}
In this section we discuss how addressing \gls{gpi} can require models with more capacity than that which would be required by \gls{cf} if the "identity" of the target process was known in advance, \eg a predictor for a single time series.

\paragraph{Linear processes} We start by considering the simple scenario of stationary linear \gls{ar} processes of order $P=2$ (which can be generalized to $P \in \sN^+$)
\begin{equation}\label{eq:linear_ar_2}
    x^n_t = \phi^n_1 \cdot x^n_{t-1} + \phi^n_2 \cdot x^n_{t-2} + \epsilon^n_t, \quad \epsilon^n_t \sim \gN\left(0, \sigma^2\right).
\end{equation}
In this case process parameters $\phi^n = \left(\phi^n_1,\phi^n_2\right)$ follow the target distribution $\boldsymbol{\gP}_{\boldsymbol{\phi}}$.
For simplicity, consider the case of a predictor with input window length $W \geq P$. The optimal output for such predictor corresponds to $\E\left[x^n_t \given \vx^n_{t-W:t}\right]$.
According to \Eqref{eq:posterior-average} we expect the optimal prediction $\E\left[x^n_t \given \vx^n_{t-W:t}\right]$ to be a non-linear function of $\vx^n_{t-W:t}$, even if the individual target processes are linear. Consider, for instance, the example in \Appref{app:ar}, where the conditional expectation involves the product of two functions of $\vx^n_{t-W:t}$ (\Eqref{eq:expansion_ar2}). However, if the target distribution $\boldsymbol{\gP}_{\boldsymbol{\phi}}$ implies only a single possible process, $\E\left[x^n_t \given \vx^n_{t-W:t}\right]$ would be linear in $\vx^n_{t-W:t}$. As such we expect a linear model to be not expressive enough to approximate $\E\left[x^n_t \given \vx^n_{t-W:t}\right]$ if multiple distinct linear \gls{ar} processes are possible under $\boldsymbol{\gP}_{\boldsymbol{\phi}}$. 

\begin{figure}[t]
    \centering
    \includegraphics[width=0.9\linewidth]{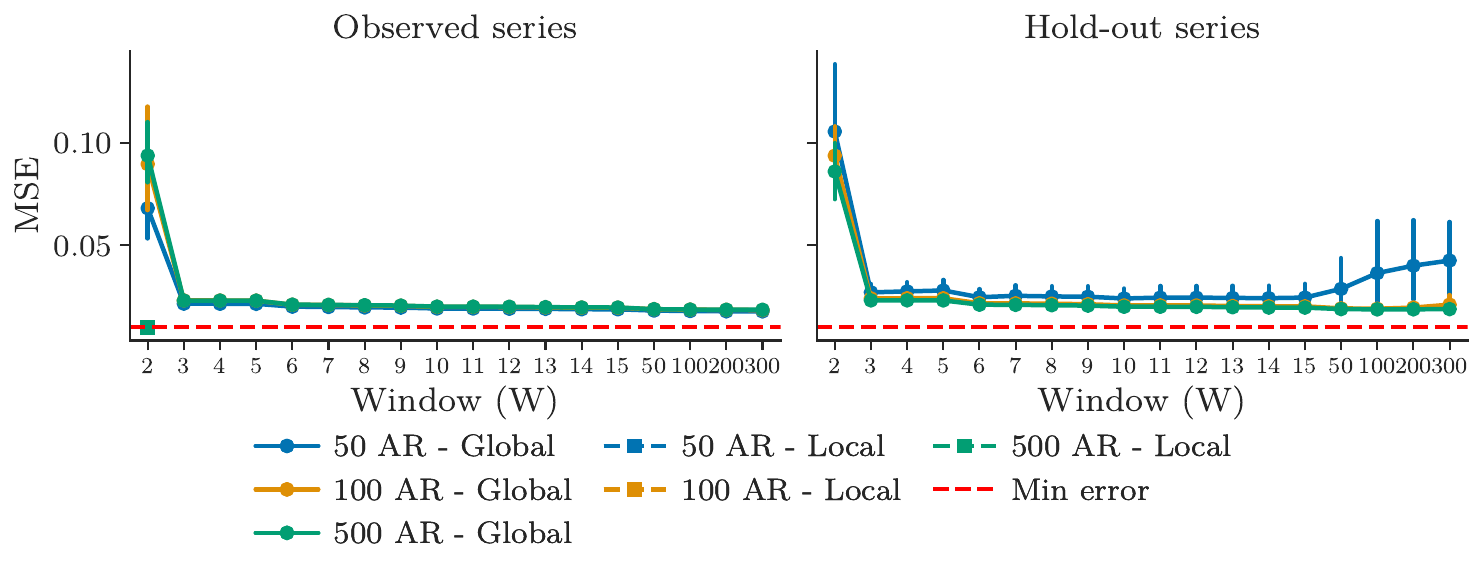}
    \caption{Performance of a linear model on data from multiple \gls{ar}(2) processes with $\sigma^2=0.1$ (\gls{mse}, 5 runs $\pm std$). For each process we generate $10000$ observations and use an $80\%-20\%$ train-test split across the temporal axis. Held-out data come from a fixed set of $200$ \gls{ar}(2).}
    \label{fig:linear_on_linear}
\end{figure}
\begin{figure}[t]
    \centering
    \includegraphics[width=0.65\linewidth]{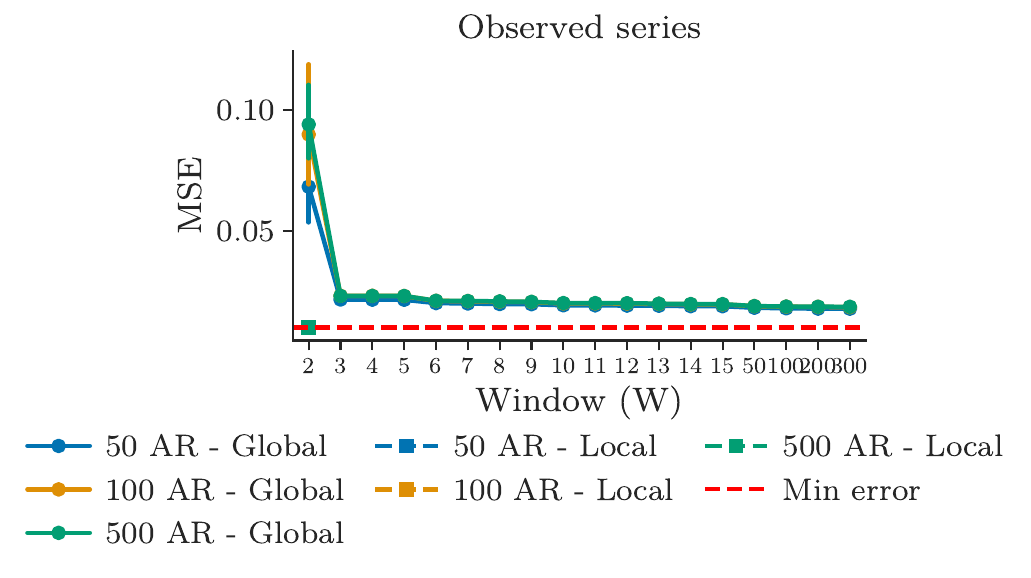}
    \caption{Performance of a linear model on data from multiple \gls{ar}(2) processes with $\sigma^2=0.1$ (\gls{mse}, 5 runs $\pm std$). For each process we generate $10000$ observations. Global models use an $80\%-20\%$ train-test split across the temporal axis. Local models are trained using only a subsequence of $300$ observations from each time series and tested on the same data as global models.}
    \label{fig:linear_on_linear_reduced}
\end{figure}
To show this effect empirically, we train\footnote{We learn the linear models by using \textit{ordinary least squares}.} both global and local linear models over a group of linear \gls{ar}(2) processes generated according to \Appref{app:synthetic_data}, considering different amounts of training processes and different input window lengths. 
The results of the experiment are reported in \Figref{fig:linear_on_linear}. As one would expect, the local approach reaches near-minimum attainable error with $W=2$ (left subplot). However, we can observe that the global one is unable to close this gap, even when the input sequence is very long ($W=300$). Increasing the number of processes in the training dataset improves inductive performance (right subplot), likely by providing a richer sample of the target distribution. However, overall the error remains higher \wrt the local approach. In line with our previous discussion, these results suggest that, even if the target family of generative processes is linear, a linear model lacks the expressivity to address \gls{gpi}. To further substantiate this, we train local models using just a subsequence of $300$ observations from each time series, corresponding the maximum input window length used for global models. \Figref{fig:linear_on_linear_reduced} shows an equivalent of the left subplot in \Figref{fig:linear_on_linear}, where local models have been learned on this reduced data amount and tested on the same data as global counterparts. Notably, local models achieve near-minimum error also in this scenario. 

\begin{figure}[t]
    \centering
    \begin{subfigure}[b]{0.49\textwidth}
        \centering
        \includegraphics[width=\textwidth]{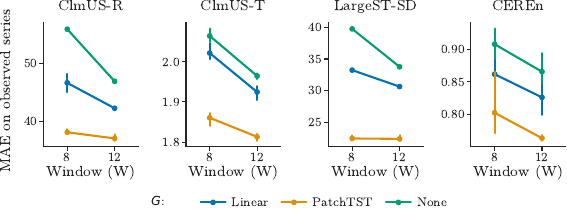}
        \caption{Transductive setting.}
        \label{fig:real_world_linear_exp_transd}
    \end{subfigure}
    \hfill
    \begin{subfigure}[b]{0.49\textwidth}
        \centering
        \includegraphics[width=\textwidth]{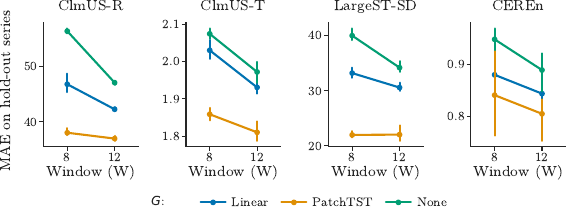}
        \caption{Inductive setting.}
        \label{fig:real_world_linear_exp_ind}
    \end{subfigure}
    \caption{24-step ahead forecasting error (\gls{mae}, 3 runs $\pm std$) for different implementations of $G\left(\cdot\right)$ (\Eqref{eq:explicit_identifier}) in the \textit{decoupled} approach of \Secref{sec:methodology}. \textit{None} refers to the \textit{standard} approach. $F\left(\cdot\right)$ (\Eqref{eq:explicit_forecaster}) is always PatchTST.}
    \label{fig:real_world_linear_exp}
\end{figure}
\paragraph{Real-world data} To expand the analysis on how addressing \gls{gpi} can require higher model capacity than \gls{cf}, we now consider real-world settings. Linear models have been shown to perform competitively \wrt complex models in transductive settings~\citep{zeng2023transformers}, but may struggle on benchmarks where the number of training time series is large~\citep{li2023revisiting}. A possible interpretation for this phenomenon is that, in many real-world tasks, a linear model may be suitable for one or few fixed time series, however they may struggle when \gls{gpi} is more complex, \eg inductive scenarios. 
To asses this we consider the decoupled modeling approach of \Secref{sec:methodology} and carry out an experiment similar to that in \Figref{fig:real_world_context_exp}. Here, however, we use a linear model, \ie a linear layer, to implement $G\left(\cdot\right)$ (\Eqref{eq:explicit_identifier}), and we compare it against implementing $G\left(\cdot\right)$ by means of PatchTST (a transformer-based neural network), or using the standard modeling approach without $G\left(\cdot\right)$ (\Eqref{eq:global_predictor}).
\Figref{fig:real_world_linear_exp} reports the results of the experiment across different datasets, when considering PatchTST as the implementation of $F\left(\cdot\right)$ (\Eqref{eq:explicit_forecaster}). 
We can observe that, while using a linear layer for $G\left(\cdot\right)$ improves performance over the standard approach, it yields consistently higher error compared to using PatchTST for $G\left(\cdot\right)$. This suggests that the complexity of \gls{gpi} could be a leading factor in determining the capacity needed by global forecasting models even in real-world settings.

\section{Datasets}\label{app:datasets}

\begin{table}[ht]
\centering
\caption{Details of the considered real-world dataset.}
\vspace{0.3cm}
\label{tab:datasets}
    \begin{tabular}{l c c c c}
    \toprule
    Dataset & \# of time series & \# of time steps & \# of channels & Sampling rate \\
    \midrule
    \gls{cer}          & 960  & 12,865 & 1  & 1 hour \\
    \gls{largest}      & 716  & 52,128 & 1  & 1 hour \\
    \gls{climate-t}    & 1000 & 26,304 & 1  & 1 hour \\
    \gls{climate-r}    & 1000 & 26,304 & 1  & 1 hour \\
    \midrule
    Electricity        & 321  & 26,304 & 1  & 1 hour \\
    Traffic            & 862  & 17,544 & 1  & 1 hour \\
    Weather            & 1   & 52,696  & 21 & 10 minutes \\
    \bottomrule
    \end{tabular}
\end{table}

\paragraph{\gls{cer}} dataset of electric load readings from~\citep{ceren}. Contains energy consumption readings from smart meters across small and medium enterprises in Ireland. Data was collected during the Smart Metering Project, supervised by the Commission for Energy Regulation (CER). Access to the data has to be requested through \url{https://www.ucd.ie/issda/data/commissionforenergyregulationcer/}. We considered a subset of the available smart meters, those for small and medium enterprises of type $3$ (SME\_CODE) with at most $1\%$ of missing readings. Data spans $\sim 1.5$ years.

\paragraph{\gls{largest}} large scale traffic dataset from~\citet{liu2023largest}. Contains traffic flow readings from loop detectors across the California state highway system. We consider the \textit{SD} sub-dataset, which contains only sensors in the San Diego area. Licensed under Attribution 4.0 International (CC BY 4.0). Data span $\sim 6$ years.

\paragraph{\gls{climate-t} and \gls{climate-r}} dataset of climate variables from satellite readings, inspired by that from \citet{de2024graph}. Contains data from the POWER Project’s Daily 2.3.5 version on 2023/02/26 and sampled from a grid of $1000$ coordinates spanning the entire surface of the United States. Visit the project's website~(\url{https://power.larc.nasa.gov/}) for further information and programmatic access to the data. In particular, \gls{climate-t} consists of hourly \textit{mean temperature} ($^{\circ}C$) readings, while \gls{climate-r} considers the \textit{allsky surface shortwave irradiance} ($W/m^2$). Data spans $3$ years.

\paragraph{Electricity} dataset of hourly electricity consumption readings from \citet{wu2021autoformer}. The dataset consists of univariate time series from $321$ customers, spanning the years from $2012$ to $2014$. The dataset is released under the MIT license.

\paragraph{Traffic} dataset of hourly road occupancy rates from \citet{wu2021autoformer}. Data is collected from the California Department of Transportation, and consists of measurements at $862$ locations across the San Francisco Bay area freeways. The data span a period of $2$ years, from $2016$ to $2018$. The dataset is released under the MIT license.

\paragraph{Weather} dataset of meteorological indicators (\eg humidity, temperature) from \citet{wu2021autoformer}. Data are recorded every $10$ minutes over the entire year of $2020$, and consists of a single multivariate time series with $21$ channels. The dataset is released under the MIT license.

\subsection{Synthetic datasets}\label{app:synthetic_data}
To build synthetic datasets, we generate time series from \gls{nar} processes of order $P$ using the following generative model:
\begin{equation}\label{eq:nar_model}
\begin{cases}
    \, x^i_t = \tanh\!\left(\sum_{p=1}^P \phi_p^i x_{t-p}^i\right) + \epsilon^i_t \\[6pt]
    \, \left[\phi^i_1,\dots,\phi^i_P\right] \sim \gU\!\left(-\tfrac{1}{\sqrt{2}}, \tfrac{1}{\sqrt{2}}\right) \\[6pt] 
    \, \epsilon^i_t \sim \gN(0, 0.2)
\end{cases}
, \; \forall \, i, t,
\end{equation}
where $\gU\left(a, b\right)$ denotes a $P$-dimensional uniform distribution over the interval $\left[a,b\right]$, while $\gN\left(\mu, \sigma^2\right)$ denotes a gaussian distribution with mean $\mu$ and variance $\sigma^2$.

When considering time series generated by \textit{linear} AR processes of order $P=2$, we drop the $\tanh$ from \eqref{eq:nar_model}.
Moreover, we sample process parameters as $\left[\phi^i_1,\phi^i_2\right] = \left[\cos\left(\alpha^i\right), \sin\left(\alpha^i\right)\right]$, with $\alpha^i \sim \gU\left(0, 2\pi\right)$, and keep only samples where $\left|\phi^i_1 - \phi^i_2\right| < 1$ and $\left(\phi^i_1 + \phi^i_2\right) < 1$. The purpose of this is to ensure the system's stability, \ie avoid collapse to $\pm \infty$, for the linear case, where the is not $\tanh$ to keep the output bounded.

\paragraph{Synthetic data scale} For the experiment in \Figref{fig:controlled_context_exp_single} we generate $10000$ time-steps for each process, across a total of $1000$ processes. Conversely, for the experiment in \Figref{fig:multitask_controlled_exp}, which involved combining multiple datasets, we generate $5000$ time-steps per process, across a total of $500$ processes per dataset. 

\subsection{Data splitting}\label{app:data_splitting}

\begin{wraptable}[10]{r}{0.6\textwidth}
\centering
\caption{Data splitting strategy.}
\label{tab:data_splitting}
\small
\resizebox{0.5\textwidth}{!}{%
\begin{tabular}{|m{0.1cm}|
                >{\centering\arraybackslash}m{4cm}
                |>{\centering\arraybackslash}m{2.5cm}
                |>{\centering\arraybackslash}m{3cm}|}
\hline
& \multicolumn{3}{c|}{\textbf{\# Timesteps}} \\ \hline
\multirow{3}{*}{\raisebox{-3.5\normalbaselineskip}[0pt][0pt]{\rotatebox{90}{\centering \textbf{\# Series}}}}
  & \parbox[c][1.25cm][c]{4cm}{\centering \cellcolor{green!30} Training Data} & \parbox[c][1cm][c]{2.5cm}{\centering \cellcolor{yellow!30} In-sample Validation}  & \parbox[c][0.7cm][c]{3cm}{\centering \cellcolor{red!30} In-sample Test} \\ \cline{2-4}
  & & \parbox[c][0.7cm][c]{2.5cm}{\centering \cellcolor{yellow!30} Out-of-sample Validation} & \\ \cline{2-4}
  & & & \parbox[c][1cm][c]{3cm}{\centering \cellcolor{red!30} Out-of-sample Test} \\ \hline
\end{tabular}
}%
\end{wraptable}
It is important to note that we split datasets for \textit{train}, \textit{validation} and \textit{test} both along the temporal axis and the 'series' axis. Meaning that we first select some time series which will be entirely held-out for \textit{inductive}, \ie out-of-sample, validation and testing. Then we take the series in the training set, and we hold-out the latest realizations for \textit{transductive}, \ie in-sample, validation and testing. Table~\ref{tab:splits_per_dataset}, reports the actual number of time series and time steps, per series, we used for each split in real-world datasets. Note that, for synthetic datasets we used a 70/10/20 train/val/test split both for the temporal and 'series' axis. Table~\ref{tab:data_splitting} visually explains the adopted splitting strategy. To reduce bias due to data splitting, experiments repeated across multiple runs have random splits along the 'series' axis controlled by the run's seed. Note that, in synthetically generated datasets, the run's seed also controls the randomized operations in data generation (\eg data-generating process parameters, noise sampling).

\begin{table}
\centering
\caption{Data splitting details for the considered real-world datasets. IS stands for \emph{in-sample}, OOS stands for \emph{out-of-sample}.}
\label{tab:splits_per_dataset}
\setlength{\tabcolsep}{5pt}
\begin{tabular}{l ccc ccc}
\toprule
 & \multicolumn{3}{c}{Number of time series} & \multicolumn{3}{c}{Number of time-steps} \\
\cmidrule(lr){2-4} \cmidrule(lr){5-7}
 & Train/IS Val/IS Test & OOS Val & OOS Test & Train & IS/OOS Val & IS/OOS Test \\
\midrule
\gls{cer}       & 672  & 96  & 192 & 8760  & 2052 & 2053 \\
\gls{largest}   & 572  & 72  & 72  & 6132  & 876  & 1752 \\
\gls{climate-t} & 700  & 100 & 200 & 17520 & 4392 & 4392 \\
\gls{climate-r} & 700  & 100 & 200 & 17520 & 4392 & 4392 \\
\midrule
Electricity     & 225  & 32  & 64  & 18414 & 2630 & 5260 \\
Traffic         & 604  & 86  & 172 & 12282 & 1754 & 3508 \\
Weather         & 21   & N/A & N/A & 36888 & 5269 & 10539\\
\bottomrule
\end{tabular}
\end{table}

Note that we use both the in-sample and out-of-sample validation sets for early-stopping, in order to avoid over-fitting the in-sample series. 
We do this for all experiments where we train a model, by taking the mean loss between the two sets, in order to avoid assigning more importance to a set over the other due to difference in number of series.

\section{Experimental details}\label{app:exp_details}

This section provides detailed information about model implementation, training and used hyper-parameters.

\subsection{Model implementation}

This section provides additional details regarding the implementation of the models we adopted in our experiments.

\subsubsection{GRU and GRU with context}\label{app:gru_with_context}

In our experiment we used the standard \gls{gru}~\citep{cho2014properties} architecture, described as
\begin{align*}
    \vr^n_{t} &= \sigma\left(\mW_{r} x^n_{t} + \vb_{r} + \mW_{r} \vh^n_{t-1} + \vb_{r}\right)) \\
    \vz^n_{t} &= \sigma\left(\mW_{z} x^n_{t} + \vb_{z} + \mW_{z} \vh^n_{t-1} + \vb_{z}\right) \\
    \vm^n_{t} &= \tanh\left(\mW_{m} x^n_{t} + \vb_{m} + \vr_{t} \odot \left(\mW_{m} \vh^n_{t-1}+ \vb_{m}\right)\right)) \\
    \vh^n_{t} &= \left(1 - \vz^n_{t}\right) \odot \vm^n_{t} + \vz^n_{t} \odot \vh^n_{t-1},
\end{align*}
where $\mW$'s and $\vb$'s denote the learnable weights of linear layers and $\sigma$ denotes the sigmoid activation function. 
In particular, note that $x^n_{t} \in \sR$ is the input signal for the $n$-th time series at time $t$, while $\vh^n_{t} \in \sR^{d_{h}}$ is the hidden state for the $n$-th time series at time $t$.
Thus we use the following shorthand notation to denote the recursive application of a \gls{gru}, of learnable parameters $\boldsymbol{\Psi}$, over a sequence $\vx^n_{t-W:t}$, with initial state $\vh_0$,
\[
\vh^n_{t} = GRU\left(\vx^n_{t-W:t}, \vh_0; \boldsymbol{\Psi}\right).
\]
Unless differently specified, $\vh_0$ is initialized to a vector of zeros.

In the experiments where only the contiguous \textit{window} over $\left[t-W,t\right)$ is present, our model can be described as
\begin{align*}
    \vh^n_{t} &= GRU\left(\vx^n_{t-W:t}, \vh_0; \boldsymbol{\Psi}\right),\\
    \hat{\vx}^n_{t:t+H} &= MLP\left(\vh^n_{t};\boldsymbol{\Phi}\right),
\end{align*}
where $MLP\left(\cdot;\boldsymbol{\Phi}\right):\sR^{d_{h}} \rightarrow \sR^{H}$ denotes a simple \gls{mlp}, with one hidden layer and hidden size $d_{h}$, mapping the last hidden state to the prediction horizon length.
Conversely, when we consider also the \textit{context} sequence, spanning $\left[t'-C, t'\right)$, with $t' < t-W$, we use a model of the following form
\begin{align*}
    \vh^n_{t'} &= GRU\left(\vx^n_{t'-C:t'}, \vh_0; \boldsymbol{\Psi}_{C}\right),\\
    \vh^n_{t} &= GRU\left(\vx^n_{t-W:t}, \vh^n_{t'}; \boldsymbol{\Psi}_{W}\right),\\
    \hat{\vx}^n_{t:t+H} &= MLP\left(\vh^n_{t};\boldsymbol{\Phi}\right).
\end{align*}
This implementation simply entails learning dedicated weights to process the \textit{window} or the \textit{context} and using the last state $\vh^n_{t'}$, obtained by processing the \textit{context}, as initial state for the \gls{gru} which will process the \textit{window}. The motivation behind this implementation is informing the model whether it is processing the latest observations preceding the target or not. Nonetheless, this can also be achieved by other means, \eg using a standard \gls{gru} and inserting a separator value between context and window.

\subsubsection{State-of-the-art architectures}\label{app:sota_architectures}

\begin{table}[ht]
\centering
\caption{Hyper-parameters used for the \gls{sota} reference architectures adopted in the paper.}
\vspace{0.3cm}
\label{tab:model_hparams}
    \begin{tabular}{l c c c c c}
    \toprule
    Model & Hidden size ($d_h$) & \# of layers & Feedforward size & \# of heads & Patch size \\
    \midrule
    PatchTST      & 16  & 3 & 128 & 4 & 16\\
    ModernTCN          & 64  & 4 & 512 & -  & -\\
    LRU           & 64 & 6 & - & - & -\\
    TSMixer    & 64 & 2 & - & -  & -\\
    \bottomrule
    \end{tabular}
\end{table}
All the \gls{sota} architectures we consider in this work follow the original paper implementation, thus, for each one we report relevant hyper-parameters (\tabref{tab:model_hparams}) alongside an explanation of how $\ve^n \in \sR^{d_e}$~(\Eqref{eq:explicit_predictor}) was given as input to $F\left(\cdot\right)$. 
Notice that in all our experiments we fix $d_e=64$.

\paragraph{PatchTST} PatchTST~\citep{nie2022time} is a standard \textit{encoder-only} transformer, which, given a sequence of length $W$, transforms it into patches of length $W_{p}$, which are then projected, through a linear layer, to size $d_h$. As such, we simply map $\ve^n \in \sR^{d_e}$ to size $d_e$, by means of a linear layer, and then sum its value to each input patch representation. For sequences that are not a multiple of the patch size we follow the original implementation which uses left-padding replicating the first value.

\paragraph{ModernTCN} ModernTCN~\citep{luo2024moderntcn} is a convolution based architecture, where the first $1D$ convolution layer maps the input signal $\vx^n_{t-W:t}$ from size $d_{x}$ to size $d_h$. Right after this initial layer, we sum $\ve^n$ to the latent representations, broadcasting along the temporal dimension. If $d_e \neq d_h$, we map it to the appropriate dimension by means of a linear layer.

\paragraph{\gls{lru}} \gls{lru}~\citep{orvieto2023resurrecting} is a recurrent architecture, so, in the same fashion as we do for the \gls{gru}~(\ref{app:gru_with_context}), we use $\ve^n$ as initial state for the recurrence. In this case, since \gls{lru} operates on complex numbers, we first map $\ve^n$ to a complex number in $\sC^{d_h}$ by means of two linear layers, one for the real and one for the imaginary part.

\paragraph{TSMixer} TSMixer~\citep{chen2023tsmixer} is the only architecture that already considered, in its design, a computation path to provide \textit{static attributes} as input to the model. Given this, we simply considered $\ve^n$ as a static attribute. We refer to the original paper for an explanation of how such attributes are integrated in the computation.

\subsubsection{Foundation models}\label{app:foundation_models}
\begin{table}[ht]
\centering
\caption{Hyper-parameters of the time series foundation models considered in the paper.}
\vspace{0.3cm}
\label{tab:foundation_models_hparams}
    \begin{tabular}{l c c c c c}
    \toprule
    Model & Hidden size & \# of layers & Feedforward size & \# of heads & Patch size \\
    \midrule
    Moirai2      & 384  & 6 & 1024 & 6 & 16\\
    TimerXL      & 1024  & 8 & 4096 & 8  & 96\\
    \bottomrule
    \end{tabular}
\end{table}
For the foundation models considered in the paper, we consider the original implementation and pre-trained weights. We report the relevant hyper-parameters in Tab.~\ref{tab:foundation_models_hparams}. Note that, when using sequence lengths that are not a multiple of the models' patch size, we adopt the padding or masking strategy adopted by the original implementations. Following is a brief description of the architectures and their pretraining data.

\paragraph{Moirai2} Moirai2~\citep{liu2025moirai} is a decoder-only Transformer with causal attention. It is pretrained on the GiftEVAL Pretrain dataset~\citep{aksu2024gift}, a large-scale dataset spanning time series from several real-world domains for a total of $\approx 230$ billion data-points (\ie individual time-steps).

\paragraph{TimerXL} TimerXL~\citep{liu2025timer} uses a decoder-only Transformer architecture with causal attention. It is pretrained on both the LOTSA dataset~\citep{woo2024unified} and the UTSD dataset~\citep{liu2024timer} for a total of $\approx 260$ billion data-points. Both datasets span time series from several real-world application domains.

\subsection{Training details}\label{app:training_details}
All models are trained using the \gls{mse} as loss function and the \emph{Adam} optimizer~\citep{kingma2014adam}, where the task is to map an input sequence of fixed length $W$ to the sequence of $H$ future values.
We train all models for $50$ epochs, with early-stopping, with patience of $8$ epochs and scheduled learning rate reduction with a factor of $0.1$ after $5$ epochs without loss improvement. 
For models trained on real-world data we fix the learning rate to $0.001$, except for \emph{ModernTCN}-based architectures, for which we set the learning rate to $0.00025$ on \gls{cer}, \gls{climate-t}, and \gls{climate-r} to avoid optimization instability.
We use a batch size of $2560$, and limit the number of train batches based on dataset size. In particular, we use $1140$ for \gls{cer}, $2880$ for \gls{largest} and $1200$ for \gls{climate-t} and \gls{climate-r}. 
For experiments with \glspl{gru} on synthetic data we use a learning rate of $0.01$.

We use sinusoidal temporal encodings with daily period for the experiments in \Secref{sec:methodology}, these encodings are commonly used to inject a notion of temporal distance among data-points in deep learning models for time series.
When \textit{context} is used (\eg \Figref{fig:foundation_context_bars}, \Figref{fig:real_world_context_exp}), the maximum back-shift allowed, during training, is $S=336$, which corresponds to context at most $2$ weeks old for the considered data with hourly frequency.
During testing the back-shift is fixed to $336$ for consistency.
The only exception is the experiment reported in \Figref{fig:controlled_context_exp_single}, where this number is set to $50$, however the choice of this parameter is not critical in such scenario. 

\paragraph{Computing resources}
We ran our experiments on A5000 and Titan V NVIDIA GPUs. 
All of our experiments fit within $24$ GBs of VRAM, but a large part of them can even run on more moderate hardware, with $12$ GBs of VRAM.
By reducing the batch size one could even stay within the very contained amount of $<8$ GBs of VRAM.
A single experiment, \ie training a single architecture on a single dataset, can take up to 4 hours, considering the largest datasets and more demanding models.
However, some experiments, like those in the controlled environments, take few minutes to be executed.
In the context of a single A5000 NVIDIA GPU, with $24$ GBs of VRAM, we estimate the total amount of compute-hours to carry out the experiments in the paper, to be $\approx 1800$h ($\approx 75$ days).

\section{Fixed context and caching}\label{app:foundation_caching}

A prominent architectural choice for foundation models is the \textit{decoder-only transformer}~\citep{vaswani2017attention}, thanks to its auto-regressive nature and the ease of handling sequences of varying length either as input or output.
This kind of architecture adopts \textit{causal-attention}, where past sequence tokens are not allowed to attend to future tokens. 
Thanks to this choice, once the key, value and query representations are computed for a token, these can be cached and re-used for future calculations, as the addition of a new subsequent token does not affect their value.

However, many foundation models for time series, \eg the considered Moirai2~\citep{woo2024unified} and TimerXL~\citep{liu2025timer}, adopt the common practice of \textit{input patching}, where the individual input token to the model is not a single time-step, but a patch of $L$ contiguous ones.
This has shown to provides several advantages~\citep{nie2022time}.
However, this means that, in the standard setting where the input to the model consists of the most recent $W$ observations, we cannot take advantage of caching, as all tokens change at each new time-step. In fact, to ensure past token would remain the same, we would have to perform inference only after $L$ new observations are collected. In contrast to domain-specific global models, which are trained for a specific target application and thus could adapt the patch size $L$ to the task's needs, foundation models cannot tune the patch size to any specific target application.

If, similarly to \Secref{sec:methodology}, we provide as input a historical sequence of length $C$, followed by a shorter window of length $W$ containing recent observations, we can take back the advantages of caching we lost due to input patching.
In fact, as the historical sequence serves for \gls{gpi}, it does not need to be updated for each new time-step, hence we can pre-compute all the relative key, value and query representations throughout all the layers in the transformer architecture. We can think of this collection of keys, values and queries as representing $\ve^n$ in \Eqref{eq:explicit_identifier}.
If $q$ is the number of fixed tokens corresponding to the historical sequence, and $d \ll q$ is the number of those that are updated with recent observations, we can reduce the computational complexity from $\bigO\left((q+d)^2\right)$ to $\bigO\left(d(q+d)\right)$, as \textit{attention} needs to be computed only between the new tokens and the historical tokens plus the new tokens.
Clearly, here the computational advantage is enabled by the ability to keep the historical sequence, and its associated tokens, fixed.

\section{Related Works}\label{app:related_works}

\paragraph{Global models.} 
Early works such as \citet{salinas2020deepar} and \citet{oreshkin2020nbeats} demonstrated the effectiveness of deep learning for time series forecasting, showing that training a single predictor across a collection of time series can outperform specialized per-series models.
\citet{montero2021principles} further showed that such \textit{global} models benefit from larger observation windows compared to \textit{local} models, as they require higher capacity to match local model's performance.
Attention-based models~\citet{vaswani2017attention} were introduced~\citep{wu2021autoformer,haoyietal2021informer,liu2022pyraformer,zhou2022fedformer}, to exploit the Transformer's ability of learning from long input sequences. Notably, \citet{nie2022time} popularized the now-standard technique of slicing input sequences into \textit{patches}, treating each patch as a token rather than individual observations.
The superior performance of these architectures has been largely attributed to their ability to capture long-range dependencies, \ie how distant past events influence future outcomes.
More recent studies, however, have shown that other architectures such as linear models~\citep{zeng2023transformers}, \glspl{mlp}~\citep{chen2023tsmixer}, and \glspl{tcn}~\citep{luo2024moderntcn} can perform on par with transformers.
Comprehensive reviews of deep learning models for time series forecasting can be found in \citet{benidis2022deep} and \citet{wang2024deep}.

\paragraph{Foundation models.} 
Recent years have seen the emergence of large-scale \textit{foundation models} for time series, trained on massive datasets spanning multiple domains. These global models exhibit strong inductive forecasting capabilities, achieving robust performance on both time series observed and not observed during training.
\citet{zhou2023one} demonstrated that pre-trained large language models can be fine-tuned for time series forecasting, while later works explored training transformer architectures from scratch on extensive time series corpora.
\citet{rasul2023lag} proposed a probabilistic decoder-only transformer, one of the few approaches that does not rely on input patching.
\citet{goswami2024moment} introduced a unified backbone capable of addressing multiple time series tasks and efficiently adapting to new series via lightweight fine-tuning.
The \textit{Moirai} family~\citep{woo2024unified,liu2025moirai} further extended the line of probabilistic foundation models with encoder-only and decoder-only architectures.
Several decoder-based models followed~\citep{das2024decoder,liu2024timer,liu2025timer}, targeting point prediction rather than probabilistic forecasting.
\citet{ansari2024chronos} drew inspiration from large language models, using a next-token prediction formulation by quantizing input patches and training via classification loss instead of regression.
A detailed overview of these approaches can be found in \citet{liang2024foundation}.

\paragraph{In-context learning.} 
In \glspl{llm}, \gls{icl} refers to the ability to improve task performance when provided with relevant examples or descriptions as part of the input. \citet{radford2019language} first demonstrated the benefits of including contextual documents, while subsequent works~\citep{chai2020description,brown2020language,wei2022finetuned} showed that textual task descriptions can further enhance performance.
The \textit{chain-of-thought} paradigm~\citep{wei2022chain}, where intermediate reasoning steps are reintroduced as input, represents another form of contextualization.
Recent studies~\citep{chen2022meta,min2022metaicl,shi2024context} have extended this principle to model training, yielding both scalability and performance gains. 
On the side of theoretical analysis, \gls{icl} has been linked to the ability of models trained on multiple tasks to implement a learning algorithm in their forward pass. In particular, \citet{xie2022an,zhang2025whatandhow} link this to Bayesian inference of latent variables characterizing the target task. \citet{akyurek2023what} discusses how transformers can implement learning algorithms like gradient descent or ordinary least squares.
Our results can be viewed through the lens of \gls{icl}: large input windows primarily supply contextual information to infer the target data-generating process, while only a smaller subset may directly influence the forecast.
Recently, some works explored \gls{icl} in the context of time series forecasting. From a theoretical perspective, \citet{li2023transformers} provides a mechanistic analysis of \gls{icl} as function learning. Their formulation includes 1-step ahead prediction of contiguous observations from dynamical systems, and they derive generalization bounds for the task. On the methodological side, \citet{faw2025context} showed that sequences from related series can enhance forecasting accuracy by enriching the model’s context. Similarly, \citet{auer2025zero} and \citet{ansari2025chronos} showed how covariate signals, \ie time series correlated to the target one, can improve accuracy, and assigned this benefit to the additional information to contextualize the task. \citet{lu2025incontext} showed one can restructure the input to a forecasting model as a series of input/output examples plus an input query and achieve better predictive performance. \citet{williams2025context}, instead, showed that textual descriptions can be used as guiding context to enhance forecasting accuracy. For a comprehensive review of \gls{icl} in \glspl{llm}, see \citet{dong2024survey}.

\paragraph{Global-local models.} 
When the set of target time series is fixed, \textit{global-local} models are an effective strategy to improve performance and lessen the computational cost, \wrt global models, by combining shared global components with series-specific local ones~\citep{butera2025on}.
\citet{nie2022time} highlighted the advantages of series-specific learnable positional encodings, while \citet{cini2023taming} demonstrated the effectiveness of learnable node embeddings in graph-based spatiotemporal forecasting.
Top entries in the M4 competition~\citep{makridakis2020m4} similarly employed global-local designs: \citet{smyl2020hybrid} integrated shared architectures with series-specific preprocessing parameters, and \citet{montero2020fforma} leveraged handcrafted features as local inputs.
Under our formulation, in purely global models, extending the input sequence can help reduce the uncertainty about the target data-generating process. In inductive settings this can improve the model's performance, playing a similar role to that of local components in transductive settings.

\paragraph{Adaptive predictors.} 
Our observations suggest that foundation models for time series should integrate two complementary components: one to infer contextual information about the observed series, and another to leverage this information for prediction.
This idea echoes the concept of \textit{adaptive neural networks}, where one sub-network modulates another’s weights or representations based on the input. Examples include \textit{fast weight programmers}~\citep{schmidhuber1992learning} and \textit{hypernetworks}~\citep{ha2017hypernetworks}.
\citet{oreshkin2021meta} theorized that, in deep forecasting architectures, early layers may act as a meta-learning outer loop~\citep{finn2017maml}, adapting representations for later predictive layers.
Although meta-learning strategies have been explored to a limited extent in time series forecasting, \eg \citet{talagala2023meta} used handcrafted features for model selection, and \citet{norton2025tailored} introduced a \textit{signal-mapper} to predict model parameters in low-data regimes, scalable adaptive predictors for large-scale foundation models remain largely unexplored.

\section{Additional results}\label{app:additional_results}

This section provides additional or complementary results for the experiments in the main paper.

\subsection{Complementary results for \Secref{sec:posterior-average}}\label{app:complement_pretrained_foundation}

\begin{figure}[t]
    \centering
    \includegraphics[width=\linewidth]{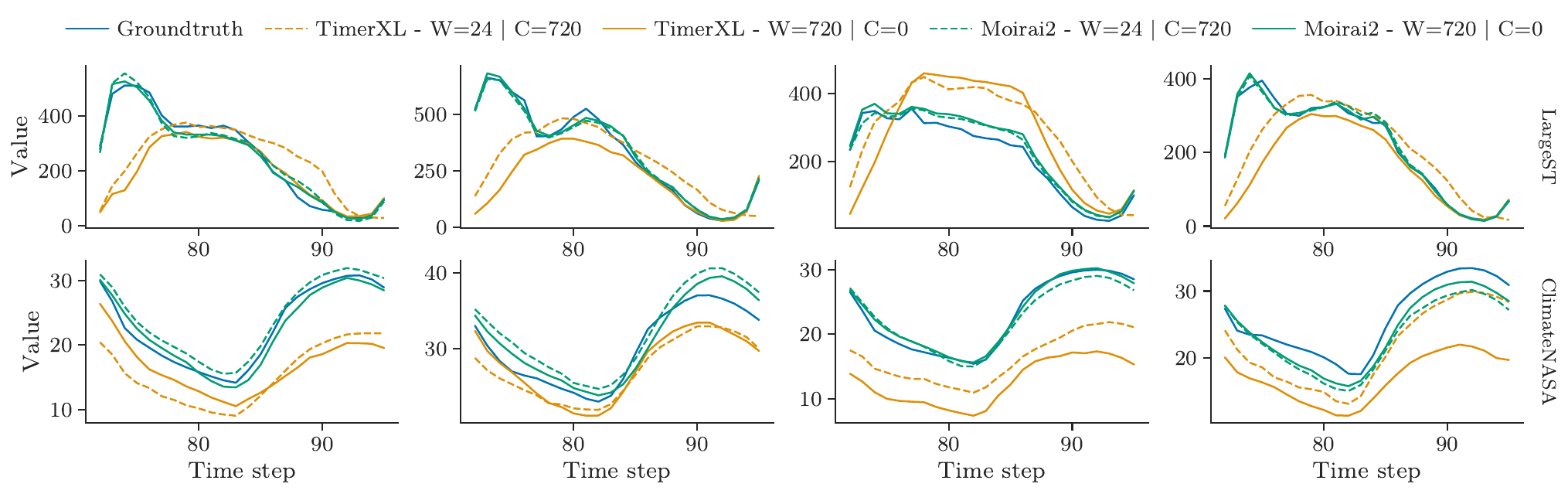}
    \caption{Example trajectories produced by foundation models (color-coded), on real-world datasets, when given as input the standard window o length $W$ (solid line) or a short window of plus context (a historical sequence) of length $C$ (dashed line).}
    \label{fig:foundation_models_trajectories}
\end{figure}

In \Figref{fig:foundation_models_trajectories} we visualize example predictions produced by the pre-trained foundation models considered in \Figref{fig:foundation_context_bars}, \ie Moirai2~\citep{woo2024unified} and TimerXL~\citep{liu2025timer}, on two real-world environments, \ie \gls{largest} and \gls{climate-t}. 
The figure visually compares the results obtained by using, as input, a standard window of length $W=720$, immediately preceding the target time-step $t$, with those obtained by concatenating a historical sequence of length $C=720$ and a short window of length $W=24$. 
We can observe how, in general, predictions have similar qualities. 
In particular we can notice that using the historical sequence sometimes leads to predictions that seems shifted upwards/downwards \wrt to those of obtained with the standard window.
As these models expect temporally contiguous input sequences, this effect might be the result of the artifact introduced by concatenating two sequences with a temporal gap between them.

\subsection{Complementary results for \ref{sec:theorem1}}\label{app:complement_empirical_evidence_theory}

\begin{figure}[t]
    \begin{subfigure}[b]{0.495\linewidth}
        \centering
        \includegraphics[width=\linewidth]{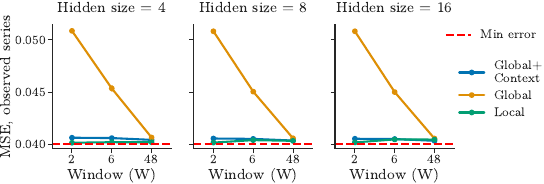}
        \caption{Transductive setting}
        \label{fig:controlled_context_exp_transd_all}
    \end{subfigure}
    \begin{subfigure}[b]{0.495\linewidth}
        \centering
        \includegraphics[width=\linewidth]{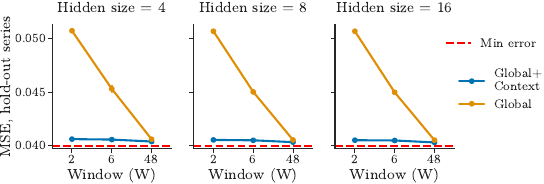}
        \caption{Inductive setting}
        \label{fig:controlled_context_exp_ind_all}
    \end{subfigure}
    \caption{1-step ahead forecasting error (\gls{mse}, 3 runs $\pm std$) for \gls{nar}(2). $C=46$. Subplots correspond to different \gls{gru} hidden size. \textit{Min error} refers to the known irreducible \gls{mse}.}
    \label{fig:controlled_context_exp_all}
\end{figure}
\begin{figure}[t]
    \begin{subfigure}[b]{0.495\linewidth}
        \centering
        \includegraphics[width=\linewidth]{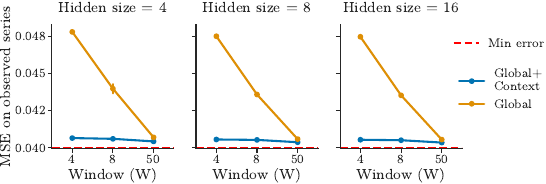}
        \caption{Transductive setting}
        \label{fig:controlled_context_exp_transd_all_order4}
    \end{subfigure}
    \begin{subfigure}[b]{0.495\linewidth}
        \centering
        \includegraphics[width=\linewidth]{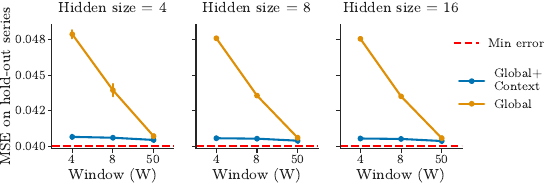}
        \caption{Inductive setting}
        \label{fig:controlled_context_exp_ind_all_order4}
    \end{subfigure}
    \caption{1-step ahead forecasting error (\gls{mse}, 3 runs $\pm std$) for \gls{nar}(4). $C=46$. Subplots correspond to different \gls{gru} hidden size. \textit{Min error} refers to the known irreducible \gls{mse}.}
    \label{fig:controlled_context_exp_all_order4}
\end{figure}
\begin{figure}[t]
    \begin{subfigure}[b]{0.495\linewidth}
        \centering
        \includegraphics[width=\linewidth]{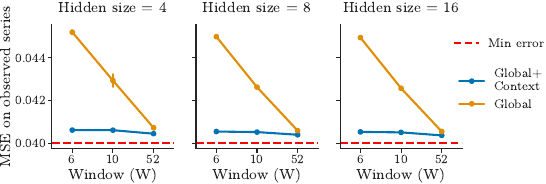}
        \caption{Transductive setting}
        \label{fig:controlled_context_exp_transd_all_order6}
    \end{subfigure}
    \begin{subfigure}[b]{0.495\linewidth}
        \centering
        \includegraphics[width=\linewidth]{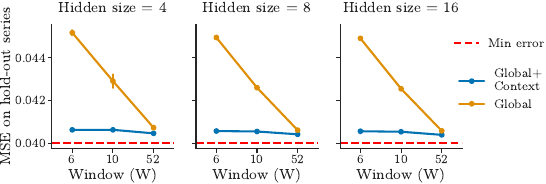}
        \caption{Inductive setting}
        \label{fig:controlled_context_exp_ind_all_order6}
    \end{subfigure}
    \caption{1-step ahead forecasting error (\gls{mse}, 3 runs $\pm std$) for \gls{nar}(6). $C=46$. Subplots correspond to different \gls{gru} hidden size. \textit{Min error} refers to the known irreducible \gls{mse}.}
    \label{fig:controlled_context_exp_all_order6}
\end{figure}
In \Figref{fig:controlled_context_exp_all} we report complementary results to those in \Figref{fig:controlled_context_exp_single}, showing how changing the model's hidden size has no significative effect on the results. 
In particular, \Figref{fig:controlled_context_exp_transd_all}, reports the error obtained on the future realizations of the time series observed during training (transductive setting).
We can observe similar results to those obtained on the entirely held-out time series (inductive setting), which are reported in \Figref{fig:controlled_context_exp_ind_all}.

In addition, \Figref{fig:controlled_context_exp_all_order4} and \Figref{fig:controlled_context_exp_all_order6} report results when repeating the experiment with \gls{nar} processes with order $P=4$ and $P=6$ respectively, where we omitted local models for simplicity. The evaluated window lengths $W$ where incremented according to the process increased temporal order, to avoid providing the model with less than $P$ observations as input. We can observe similar results to the case $P=2$, where providing the additional context sequence improves the model performance already when $W=P$. Notably, in this case, the predictors gap from the minimum attainable error sees a relative increase with the increment of $P$, possibly stemming from the need for additional input observations to infer the process generating the input.

\subsection{Complementary results for \Secref{sec:heterogeneous_vs_window}}\label{app:complement_heterogeneous_vs_window}

\begin{figure}[t]
    \centering
    \includegraphics[width=\linewidth]{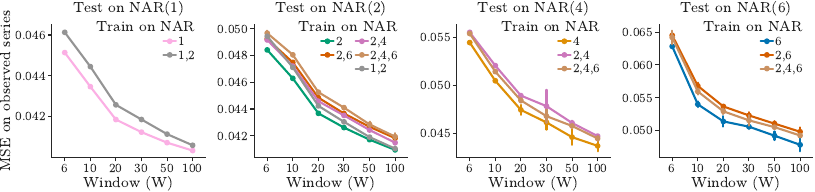}
    \caption{1-step ahead forecasting (\gls{mse}, \textit{transductive}, 3 runs $\pm std$) for different combinations of \gls{nar} domains. Columns show error on a specific domain, color identifies the training domains.}
    \label{fig:multitask_controlled_exp_transd}
\end{figure}

\Figref{fig:multitask_controlled_exp_transd} reports complementary \wrt \Figref{fig:multitask_controlled_exp}.
While the latter shows the error on held-out series (inductive setting), the former shows the error on the future, unobserved, realizations of the time series used for training (transductive setting).
Overall, we can observe a similar behavior between the two different testing scenarios.

\subsection{Complementary results for \Secref{sec:methodology}}\label{app:complement_empirical_evidence_design}

\begin{figure}[t]
    \centering
    \includegraphics[width=\linewidth]{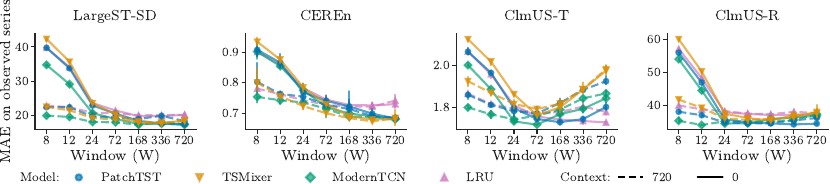}
    \caption{24-step ahead forecasting error (\gls{mae}, \textit{transductive}, 3 runs $\pm std$). Color/marker denotes the model. Line-style denotes the approach: decoupling (dashed) or standard (solid).}
    \label{fig:real_world_context_exp_transd}
\end{figure}

Here, we report complementary results showing the error on the future, unobserved, realizations of the time series used for training (transductive setting), whereas the main paper reports the error on held-out series (inductive setting).
In particular, \Figref{fig:real_world_context_exp_transd} is complementary to \Figref{fig:real_world_context_exp}.
We can see the presence of similar trends, further validating our analysis in the main paper.

\begin{figure}[t]
    \centering
    \includegraphics[width=\linewidth]{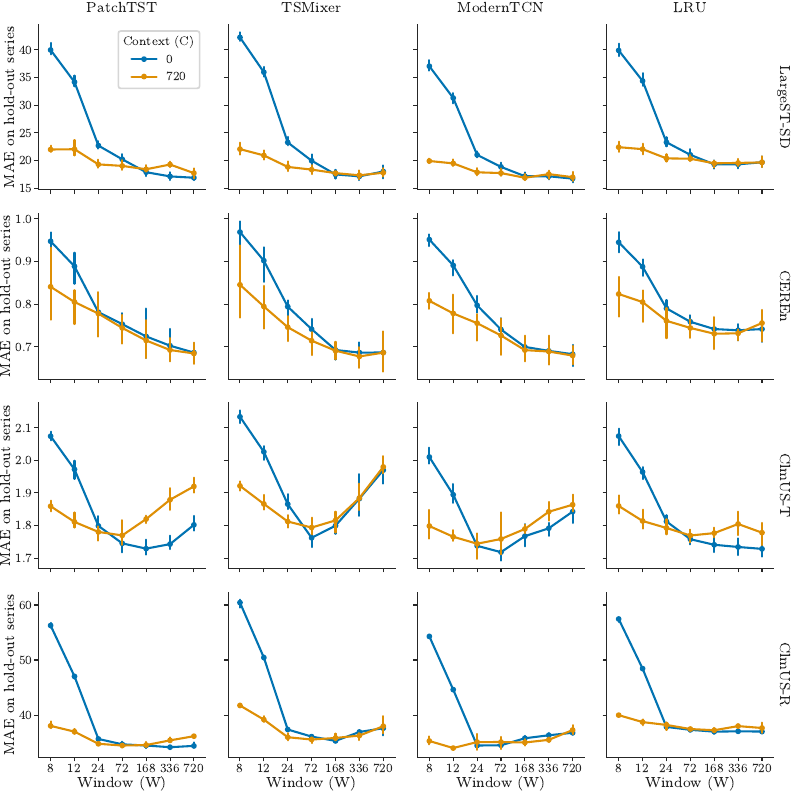}
    \caption{24-step ahead forecasting error (\gls{mae}, \textit{inductive}, 3 runs $\pm std$). Columns correspond to different models while rows represent different datasets. Line-color denotes decoupling (orange) or not (blue).}
    \label{fig:real_world_context_exp_full}
\end{figure}
\begin{figure}[t]
    \centering
    \includegraphics[width=\linewidth]{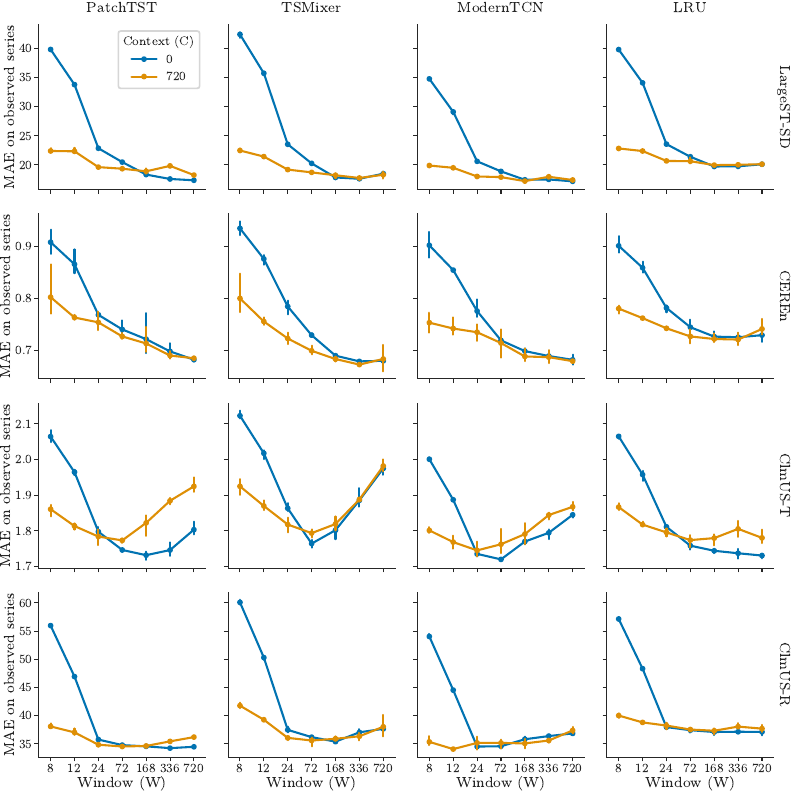}
    \caption{24-step ahead forecasting error (\gls{mae}, \textit{transductive}, 3 runs $\pm std$). Columns correspond to different models while rows represent different datasets. Line-color denotes decoupling (orange) or not (blue).}
    \label{fig:real_world_context_exp_full_transd}
\end{figure}

In addition, we report expanded variants of \Figref{fig:real_world_context_exp_transd} and \Figref{fig:real_world_context_exp}, respectively in \Figref{fig:real_world_context_exp_full_transd} and \Figref{fig:real_world_context_exp_full}.
These larger figures report separate sub-plots for each model and dataset pair, in order to enhance readability.

\begin{figure}[t]
    \centering
    \includegraphics[width=\linewidth]{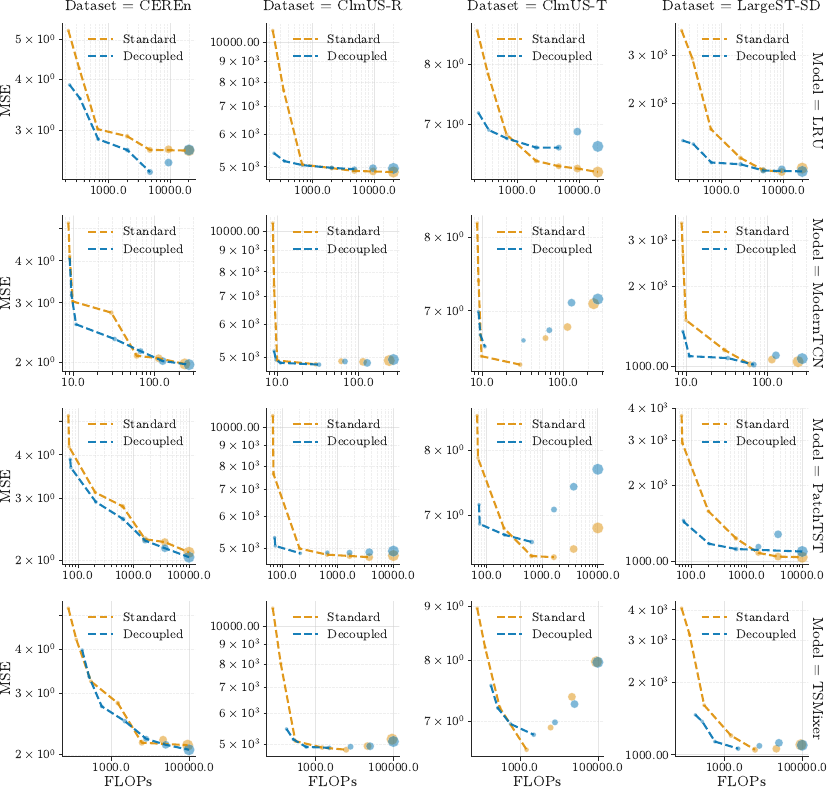}
    \caption{Pareto frontiers, \textit{inductive} \gls{mse} (y-axis) vs \gls{flops} (MFLOPs, x-axis), for the decoupled (blue) and standard (orange) approach. Dots correspond to model/dataset pairs from \Figref{fig:real_world_context_exp} (3 runs ±std). Dot size is proportional to the input window length $W$. Rows correspond to different models, columns to different datasets.}
    \label{fig:pareto_flops}
\end{figure}
\begin{figure}[t]
    \centering
    \includegraphics[width=\linewidth]{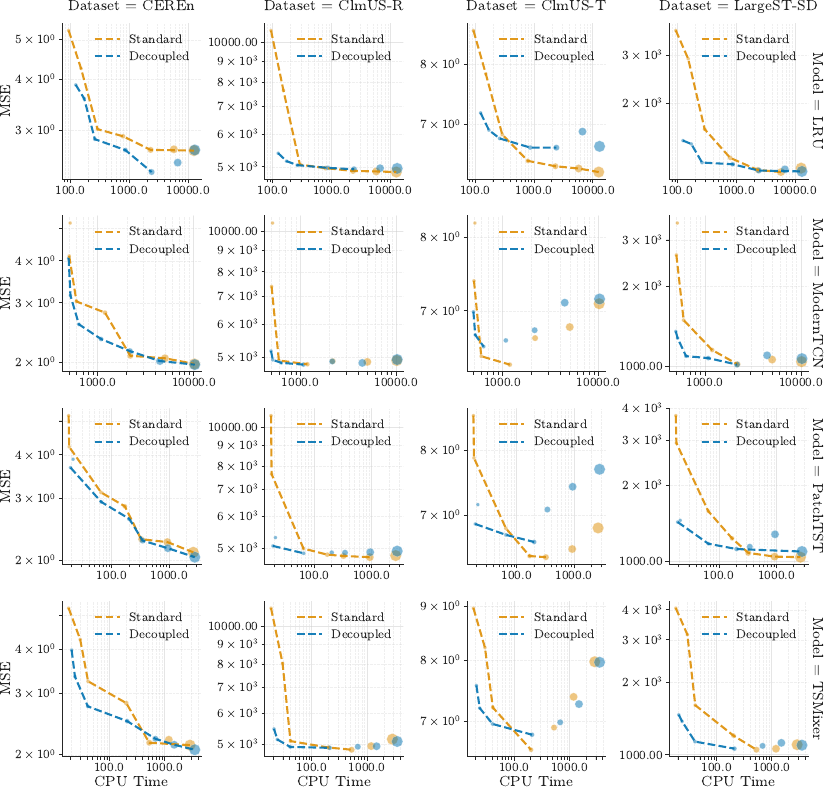}
    \caption{Pareto frontiers, \textit{inductive} \gls{mse} (y-axis) vs CPU inference time (ms, x-axis), for the decoupled (blue) and standard (orange) approach. Dots correspond to model/dataset pairs from \Figref{fig:real_world_context_exp} (3 runs ±std). Dot size is proportional to the input window length $W$. Rows correspond to different models, columns to different datasets.}
    \label{fig:pareto_cputime}
\end{figure}
\begin{figure}[t]
    \centering
    \includegraphics[width=\linewidth]{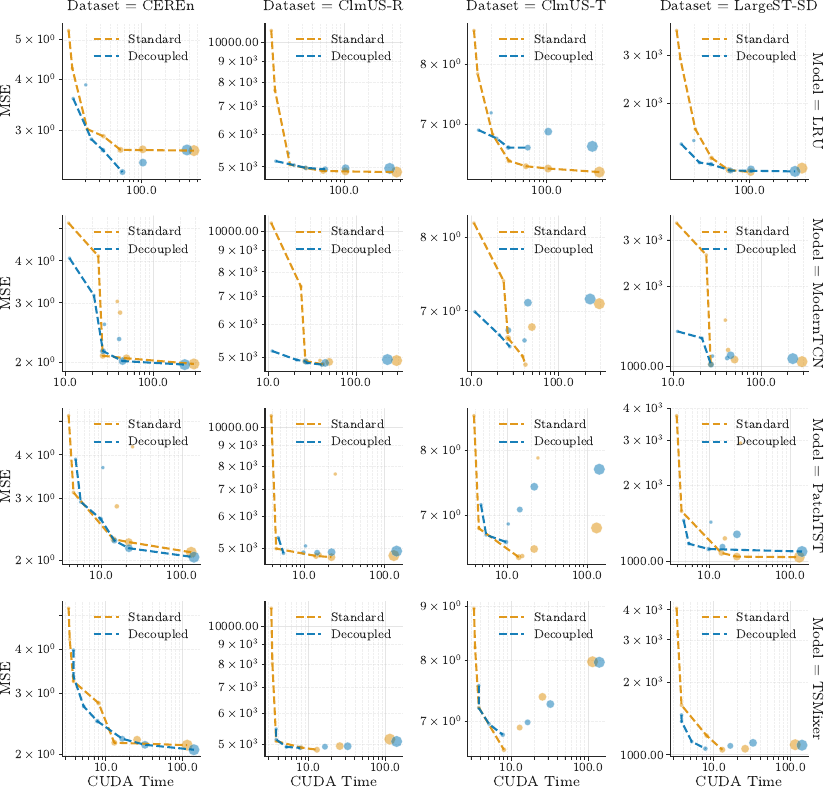}
    \caption{Pareto frontiers, \textit{inductive} \gls{mse} (y-axis) vs GPU inference time (ms, x-axis), for the decoupled (blue) and standard (orange) approach. Dots correspond to model/dataset pairs from \Figref{fig:real_world_context_exp} (3 runs ±std). Dot size is proportional to the input window length $W$. Rows correspond to different models, columns to different datasets.}
    \label{fig:pareto_cudatime}
\end{figure}
\begin{figure}[t]
    \centering
    \includegraphics[width=\linewidth]{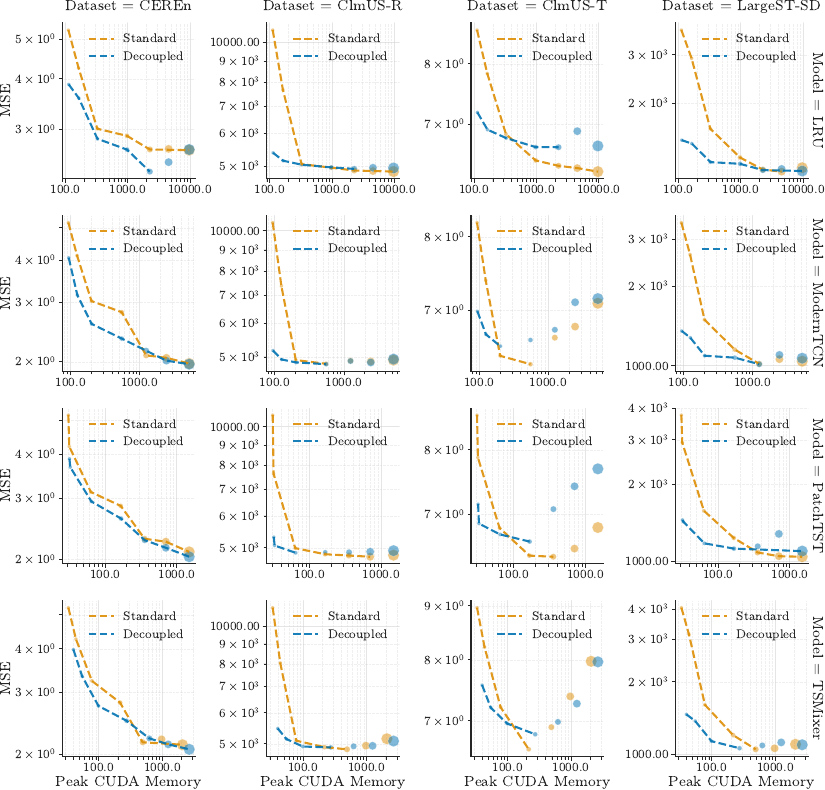}
    \caption{Pareto frontiers, \textit{inductive} \gls{mse} (y-axis) vs peak GPU memory occupancy (MB, x-axis), for the decoupled (blue) and standard (orange) approach. Dots correspond to model/dataset pairs from \Figref{fig:real_world_context_exp} (3 runs ±std). Dot size is proportional to the input window length $W$. Rows correspond to different models, columns to different datasets.}
    \label{fig:pareto_memory}
\end{figure}

Similarly, to complement \Figref{fig:pareto_merged}, we report performance-cost Pareto frontiers for individual architectures and datasets. \Figref{fig:pareto_flops} reports \gls{flops}, \Figref{fig:pareto_cputime} inference time on CPU, \Figref{fig:pareto_cudatime} inference time on GPU and \Figref{fig:pareto_memory} the peak GPU memory occupancy. Notably, results are similar to the aggregated plot in \Figref{fig:pareto_merged}. In most scenario the decoupled approach dominates the standard one at lower computational costs, which correspond to shorter input windows. As one would expect, given our analysis on the effect of \gls{gpi}, at higher costs (larger $W$) the Pareto frontiers often align.

\paragraph{Speedup within accuracy loss budget}\label{app:speedup_budget}
\begin{table}[t]
\centering
\caption{Computational advantage for standard and decoupled approaches when allowing up to 10\% relative \gls{mae} performance cost \wrt the best model using the standard approach. The table shows the speedup factors for CPU time, CUDA time, FLOPs, and peak CUDA memory. Best results are in bold.}
\label{tab:speedup_on_threshold_0.1_train_mse_test_inductive_mae}
\resizebox{0.8\linewidth}{!}{%
\begin{tabular}{llc|cccc}
\toprule
Model & Dataset & Approach & CPU T ($\times$) & CUDA T ($\times$) & FLOPs ($\times$) & CUDA Mem. ($\times$) \\
\midrule
\multirow[t]{8}{*}{LRU} & \multirow[t]{2}{*}{\gls{cer}} & Standard & 19.88 & 4.95 & 13.98 & 13.94 \\
 &  & \textbf{Decoupled} & \textbf{33.50} & \textbf{7.43} & \textbf{27.05} & \textbf{27.62} \\
\cline{2-7}
 & \multirow[t]{2}{*}{\gls{climate-r}} & Standard & 8.00 & 2.61 & 6.99 & 6.97 \\
 &  & \textbf{Decoupled} & \textbf{18.69} & \textbf{2.73} & \textbf{19.93} & \textbf{20.01} \\
\cline{2-7}
 & \multirow[t]{2}{*}{\gls{climate-t}} & Standard & 43.61 & 21.70 & 29.96 & 29.85 \\
 &  & \textbf{Decoupled} & \textbf{101.91} & \textbf{22.69} & \textbf{85.41} & \textbf{85.71} \\
\cline{2-7}
 & \multirow[t]{2}{*}{\gls{largest}} & Standard & 2.94 & 1.64 & 2.33 & 2.32 \\
 &  & \textbf{Decoupled} & \textbf{8.86} & \textbf{2.32} & \textbf{6.88} & \textbf{6.94} \\
\cline{1-7} \cline{2-7}
\multirow[t]{8}{*}{ModernTCN} & \multirow[t]{2}{*}{\gls{cer}} & Standard & 8.55 & 6.96 & \textbf{8.06} & \textbf{9.45} \\
 &  & Decoupled & \textbf{9.39} & \textbf{7.14} & 7.32 & 9.43 \\
\cline{2-7}
 & \multirow[t]{2}{*}{\gls{climate-r}} & Standard & 1.00 & 1.00 & 1.00 & 1.00 \\
 &  & \textbf{Decoupled} & \textbf{1.21} & \textbf{3.51} & \textbf{1.09} & \textbf{2.13} \\
\cline{2-7}
 & \multirow[t]{2}{*}{\gls{climate-t}} & Standard & 1.98 & 1.08 & 2.98 & 2.74 \\
 &  & \textbf{Decoupled} & \textbf{2.39} & \textbf{3.78} & \textbf{3.25} & \textbf{5.85} \\
\cline{2-7}
 & \multirow[t]{2}{*}{\gls{largest}} & Standard & 4.66 & \textbf{11.00} & 3.92 & 4.19 \\
 &  & \textbf{Decoupled} & \textbf{16.00} & 10.51 & \textbf{21.85} & \textbf{25.75} \\
\cline{1-7} \cline{2-7}
\multirow[t]{8}{*}{PatchTST} & \multirow[t]{2}{*}{\gls{cer}} & \textbf{Standard} & \textbf{15.49} & 8.41 & \textbf{15.62} & \textbf{9.02} \\
 &  & Decoupled & 12.94 & \textbf{13.38} & 15.49 & 8.96 \\
\cline{2-7}
 & \multirow[t]{2}{*}{\gls{climate-r}} & Standard & 14.22 & \textbf{4.87} & 17.95 & 11.11 \\
 &  & \textbf{Decoupled} & \textbf{48.36} & 2.10 & \textbf{49.25} & \textbf{21.54} \\
\cline{2-7}
 & \multirow[t]{2}{*}{\gls{climate-t}} & Standard & 4.96 & \textbf{3.13} & 7.91 & 5.62 \\
 &  & \textbf{Decoupled} & \textbf{15.24} & 2.95 & \textbf{22.32} & \textbf{11.20} \\
\cline{2-7}
 & \multirow[t]{2}{*}{\gls{largest}} & \textbf{Standard} & \textbf{8.09} & \textbf{9.28} & \textbf{6.18} & \textbf{4.18} \\
 &  & Decoupled & 7.61 & 9.01 & 6.16 & 4.17 \\
\cline{1-7} \cline{2-7}
\multirow[t]{8}{*}{TSMixer} & \multirow[t]{2}{*}{\gls{cer}} & Standard & 5.88 & 3.17 & 14.56 & 4.49 \\
 &  & \textbf{Decoupled} & \textbf{29.11} & \textbf{4.94} & \textbf{37.64} & \textbf{9.82} \\
\cline{2-7}
 & \multirow[t]{2}{*}{\gls{climate-r}} & \textbf{Standard} & 12.76 & \textbf{3.39} & \textbf{20.18} & \textbf{6.29} \\
 &  & Decoupled & \textbf{12.90} & 2.50 & 10.68 & 4.96 \\
\cline{2-7}
 & \multirow[t]{2}{*}{\gls{climate-t}} & Standard & 4.90 & \textbf{2.12} & 4.88 & 2.78 \\
 &  & \textbf{Decoupled} & \textbf{9.72} & \textbf{2.12} & \textbf{8.21} & \textbf{5.42} \\
\cline{2-7}
 & \multirow[t]{2}{*}{\gls{largest}} & Standard & 2.26 & 1.98 & 3.52 & 1.98 \\
 &  & \textbf{Decoupled} & \textbf{29.11} & \textbf{4.94} & \textbf{37.64} & \textbf{9.82} \\
\cline{1-7} \cline{2-7}
\bottomrule
\end{tabular}
}
\end{table}
To complement the reported Pareto frontiers, we also provide a different perspective on the performance-cost trade-off. For each dataset and model reported in \Figref{fig:real_world_context_exp}, we consider, as reference, the minimum \gls{mae} achieved by the \textit{standard} approach across different input lengths $W$. Then, for both the standard and \textit{decoupled} approach, we compute the attainable improvement factors in computational performance when allowing a penalty of up to $10$\% of the reference \gls{mae}. \Tabref{tab:speedup_on_threshold_0.1_train_mse_test_inductive_mae} reports the computational improvement factors in terms of \gls{flops}, inference time on CPU, inference time on GPU and peak GPU memory occupancy. Notably, in most scenarios the decoupled approach allows significant scalability  within this $10$\% accuracy budget.

\paragraph{Additional time series benchmarks}
\begin{figure}[t]
    \begin{subfigure}[b]{0.609\linewidth}
        \centering
        \includegraphics[width=\linewidth]{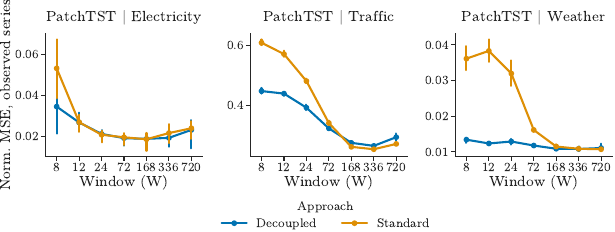}
        \caption{Transductive setting}
        \label{fig:bench_decoupled_transductive}
    \end{subfigure}
    \begin{subfigure}[b]{0.39\linewidth}
        \centering
        \includegraphics[width=\linewidth]{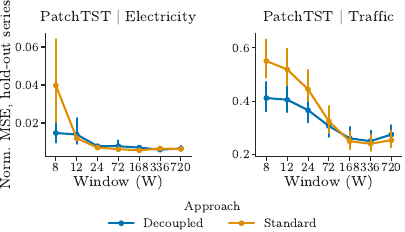}
        \caption{Inductive setting}
        \label{fig:bench_decoupled_inductive}
    \end{subfigure}
    \caption{96-step ahead forecasting error (normalized \gls{mse}, 3 runs $\pm std$) for \emph{electricity}, \emph{traffic} and \emph{weather} datasets. Standard (orange) vs decoupled (blue) approach. $C=720$.}
    \label{fig:bench_decoupled}
\end{figure}
We conduct experiments comparing the standard and decoupled approach introduced in \Secref{sec:methodology} on the \emph{electricity}, \emph{traffic} and \emph{weather} datasets~\citep{wu2021autoformer}, considering the case where both the \emph{base model} (\Eqref{eq:explicit_forecaster}) and the \emph{embedding module} (\Eqref{eq:explicit_identifier}) are implemented by PatchTST. 
\Figref{fig:bench_decoupled} reports the results for the transductive~(\Figref{fig:bench_decoupled_transductive}) and inductive~(\Figref{fig:bench_decoupled_inductive}) settings. Results are consistent with our analysis and the results observed in \Figref{fig:real_world_context_exp}. The advantage provided by the decoupled approach depends on the target time series distribution, \eg on the \emph{electricity} dataset the advantage is only for the smallest input window considered, possibly as few observation are already sufficient for \gls{gpi} in this task. Note that for \emph{weather} we report only the transductive performance as the dataset involves only $21$ time series measuring different weather variables, hence the data are not appropriate to consider an hold-out set on which to measure inductive performance.

\subsection{Complementary results for Example~\ref{prop:single_ts_case_linear}}\label{app:complement_theory_ar}

\begin{figure}[t]
    \centering
    \includegraphics[width=\linewidth]{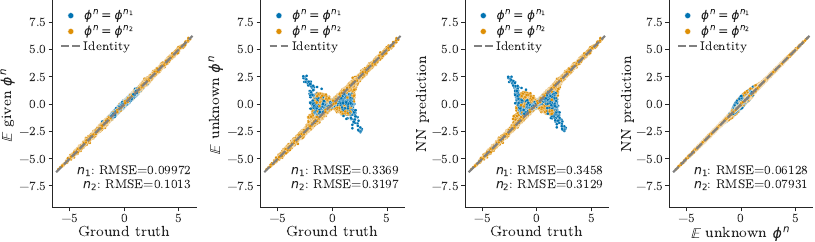}
    \caption{Parity plots for Example~\ref{prop:single_ts_case_linear}. Data generation hyper-parameters: $\sigma^2=0.1$, $seed=19$, $\#timesteps=100000$. \textit{Ground truth} denotes the actual realizations from the stochastic process. $\E$ \textit{given} $\boldsymbol{\phi}^n$ denotes the expected value for the future observation when knowing the process parameters and the last $2$ observations. $\E$ \textit{unknown} $\boldsymbol{\phi}^n$ denotes the predictions obtained by using the real process parameters weighted by the likelihood in \Eqref{eq:process_likelihood_ar2} computed in closed form. \textit{NN prediction} denotes the predictions obtained by a neural network optimized to minimize the \gls{mse}. The \gls{rmse} is computed between the two variables on each plot's axes, and reported separately for processes with $\boldsymbol{\phi}^n=\boldsymbol{\phi}^{n_1}$ and $\boldsymbol{\phi}^n=\boldsymbol{\phi}^{n_2}$.}
    \label{fig:parity_theory_ar_2}
\end{figure}
\begin{figure}[t]
    \centering
    \includegraphics[width=\linewidth]{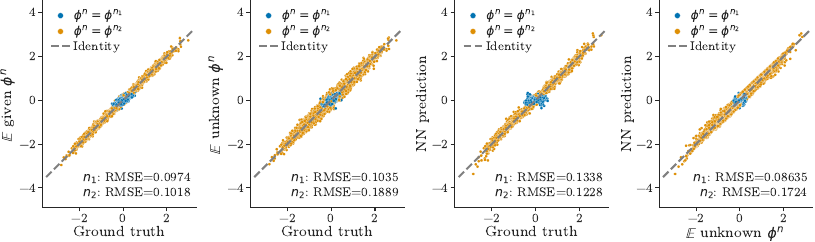}
    \caption{Parity plots for Example~\ref{prop:single_ts_case_linear}. Data generation hyper-parameters: $\sigma^2=0.1$, $seed=42$, $\#timesteps=10000$. \textit{Ground truth} denotes the actual realizations from the stochastic process. $\E$ \textit{given} $\boldsymbol{\phi}^n$ denotes the expected value for the future observation when knowing the process parameters and the last $2$ observations. $\E$ \textit{unknown} $\boldsymbol{\phi}^n$ denotes the predictions obtained by using the real process parameters weighted by the likelihood in \Eqref{eq:process_likelihood_ar2} computed in closed form. \textit{NN prediction} denotes the predictions obtained by a neural network optimized to minimize the \gls{mse}. The \gls{rmse} is computed between the two variables on each plot's axes, and reported separately for processes with $\boldsymbol{\phi}^n=\boldsymbol{\phi}^{n_1}$ and $\boldsymbol{\phi}^n=\boldsymbol{\phi}^{n_2}$.}
    \label{fig:parity_theory_ar_3}
\end{figure}
\begin{figure}[t]
    \centering
    \includegraphics[width=\linewidth]{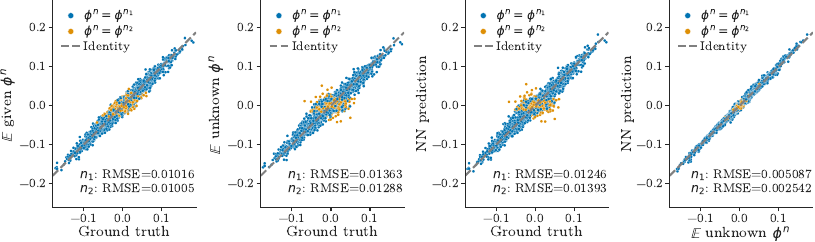}
    \caption{Parity plots for Example~\ref{prop:single_ts_case_linear}. Data generation hyper-parameters: $\sigma^2=0.01$, $seed=7$, $\#timesteps=10000$. \textit{Ground truth} denotes the actual realizations from the stochastic process. $\E$ \textit{given} $\boldsymbol{\phi}^n$ denotes the expected value for the future observation when knowing the process parameters and the last $2$ observations. $\E$ \textit{unknown} $\boldsymbol{\phi}^n$ denotes the predictions obtained by using the real process parameters weighted by the likelihood in \Eqref{eq:process_likelihood_ar2} computed in closed form. \textit{NN prediction} denotes the predictions obtained by a neural network optimized to minimize the \gls{mse}. The \gls{rmse} is computed between the two variables on each plot's axes, and reported separately for processes with $\boldsymbol{\phi}^n=\boldsymbol{\phi}^{n_1}$ and $\boldsymbol{\phi}^n=\boldsymbol{\phi}^{n_2}$.}
    \label{fig:parity_theory_ar_4}
\end{figure}
\begin{figure}[t]
    \centering
    \includegraphics[width=\linewidth]{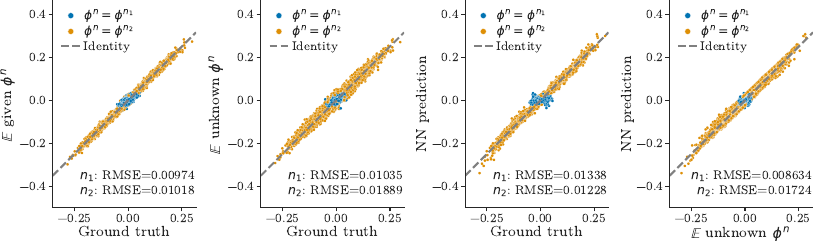}
    \caption{Parity plots for Example~\ref{prop:single_ts_case_linear}. Data generation hyper-parameters: $\sigma^2=0.01$, $seed=42$, $\#timesteps=10000$. \textit{Ground truth} denotes the actual realizations from the stochastic process. $\E$ \textit{given} $\boldsymbol{\phi}^n$ denotes the expected value for the future observation when knowing the process parameters and the last $2$ observations. $\E$ \textit{unknown} $\boldsymbol{\phi}^n$ denotes the predictions obtained by using the real process parameters weighted by the likelihood in \Eqref{eq:process_likelihood_ar2} computed in closed form. \textit{NN prediction} denotes the predictions obtained by a neural network optimized to minimize the \gls{mse}. The \gls{rmse} is computed between the two variables on each plot's axes, and reported separately for processes with $\boldsymbol{\phi}^n=\boldsymbol{\phi}^{n_1}$ and $\boldsymbol{\phi}^n=\boldsymbol{\phi}^{n_2}$.}
    \label{fig:parity_theory_ar_5}
\end{figure}
To complement the results shown in \Figref{fig:parity_theory_ar} we provide, in Figures~\ref{fig:parity_theory_ar_2},~\ref{fig:parity_theory_ar_3},~\ref{fig:parity_theory_ar_4} and~\ref{fig:parity_theory_ar_5}, results obtained under different data generation hyper-parameters. Specifically we change the random seed to sample the two \gls{ar}(2) process parameters, the value of $\sigma^2$, and the number of generated time steps. Results are similar, showing alignment between the learned neural network predictions and those of the closed-form optimal predictor for an input window of length $W=2$. 

\end{document}